\documentclass[runningheads]{llncs}
\usepackage{graphicx}

\usepackage{tikz}
\usepackage{comment}
\usepackage{amsmath,amssymb} 
\usepackage{color}

\usepackage[accsupp]{axessibility}  

\usepackage[utf8]{inputenc} 
\usepackage[T1]{fontenc}    
\usepackage{url}            
\usepackage{booktabs}       
\usepackage{amsfonts}       
\usepackage{nicefrac}       
\usepackage{microtype}      
\usepackage{xcolor}         
\usepackage{bm,bbm}
\usepackage{algorithm,algorithmic}

\usepackage{multirow}
\usepackage{{graphicx}}
\usepackage{listings}
\usepackage{rotating}  
\usepackage{float}
\usepackage{xr}
\usepackage[pagebackref,breaklinks,colorlinks]{hyperref}

\usepackage{xr}	
\usepackage[capitalize]{cleveref}
\crefname{section}{Sec.}{Secs.}
\Crefname{section}{Section}{Sections}
\Crefname{table}{Table}{Tables}
\crefname{table}{Tab.}{Tabs.}

\begin{document}
\pagestyle{headings}
\mainmatter
\def\ECCVSubNumber{1670}  

\title{KVT: $k$-NN Attention for Boosting Vision Transformers} 

\titlerunning{ECCV-22 submission ID \ECCVSubNumber} 
\authorrunning{ECCV-22 submission ID \ECCVSubNumber} 
\author{Anonymous ECCV submission}
\institute{Paper ID \ECCVSubNumber}

\titlerunning{KVT: $k$-NN Attention for Boosting Vision Transformers}
%
\newcommand*\samethanks[1][\value{footnote}]{\footnotemark[#1]}
\authorrunning{Pichao Wang et al.}
\author{Pichao Wang\thanks{The first two authors contribute equally.} \and
Xue Wang\samethanks \and
Fan Wang \and
Ming Lin \and
Shuning Chang \and
Hao Li \and
Rong Jin
}
%
%
\institute{Alibaba Group\\ \email{\{pichao.wang, xue.w, fan.w, ming.l, shuning.csn, lihao.lh, jinrong.jr\}@alibaba-inc.com}\\}
\maketitle

\begin{abstract}
    Convolutional Neural Networks (CNNs) have dominated computer vision for years, due to its ability in capturing locality and translation invariance. Recently, many vision transformer architectures have been proposed and they show promising performance. A key component in vision transformers is the fully-connected self-attention which is more powerful than CNNs in modelling long range dependencies. However, since the current dense self-attention uses all image patches (tokens) to compute attention matrix, it may neglect locality of images patches and involve noisy tokens (e.g., clutter background and occlusion), leading to a slow training process and potential degradation of performance. To address these problems, we propose the $k$-NN attention for boosting vision transformers. Specifically, instead of involving all the tokens for attention matrix calculation, we only select the top-$k$ similar tokens from the keys for each query to compute the attention map. The proposed $k$-NN attention naturally inherits the local bias of CNNs without introducing convolutional operations, as nearby tokens tend to be more similar than others. In addition, the $k$-NN attention allows for the exploration of long range correlation and at the same time filters out irrelevant tokens by choosing the most similar tokens from the entire image. Despite its simplicity, we verify, both theoretically and empirically, that $k$-NN attention is powerful in speeding up training and distilling noise from input tokens. Extensive experiments are conducted by using 11 different vision transformer architectures to verify that the proposed $k$-NN attention can work with any existing transformer architectures to improve its prediction performance. The codes are available at \url{https://github.com/damo-cv/KVT}.

\end{abstract}

\section{Introduction}
 Traditional CNNs  provide state of the art performance in vision tasks, due to its ability in capturing locality and translation invariance, while transformer~\cite{vaswani2017attention} is the de-facto standard  for natural language processing (NLP) tasks thanks to its advantages in modelling long-range dependencies. Recently, various vision transformers~\cite{dosovitskiy2020image,touvron2020training,yuan2021tokens,wang2021pyramid,han2021transformer,yuan2021incorporating,liu2021swin,wu2021cvt,touvron2021going,yuan2021volo} have been proposed by building pure or hybrid transformer models for visual tasks. Inspired by the transformer scaling success in NLP tasks, vision transformer converts an image into a sequence of image patches (tokens), with each patch encoded into a vector. Since self-attention in the transformer is position agnostic, different positional encoding methods~\cite{dosovitskiy2020image,chu2021we,d2021convit} have been developed, and in~\cite{wu2021cvt,Zhengsu21} their roles have been replaced by convolutions. Afterwards, all tokens are fed into stacked transformer encoders for feature learning, with an extra $CLS$ token~\cite{dosovitskiy2020image,touvron2020training,d2021convit} or global average pooling (GAP)~\cite{liu2021swin,Zhengsu21} for final feature representation. Compared with CNNs, transformer-based models explicitly exploit global dependencies and demonstrate comparable, sometimes even better, results than highly optimised CNNs~\cite{he2016deep,tan2019efficientnet}.

Albeit achieving its initial success, vision transformers suffer from slow training. One of the key culprits is the fully-connected self-attention, which takes all the tokens to calculate the attention map. The dense attention not only neglects the locality of images patches, an important feature of CNNs, but also involves noisy tokens into the computation of self-attention, especially in the situations of cluttered background and occlusion. Both issues can slow down the training significantly~\cite{cordonnier2019relationship,d2021convit}. Recent works~\cite{yuan2021incorporating,wu2021cvt,Zhengsu21} try to mitigate this problem by introducing convolutional operators into vision transformers. Despite encouraging results, these studies fail to resolve the problem fundamentally from the transformer structure itself, limiting their success. In this study, we address the challenge by directly attacking its root cause, i.e. the fully-connected self-attention. 

To this end, we propose the $k$-NN attention to replace the fully-connected attention. Specifically, we do not use all the tokens for attention matrix calculation, but only select the top-$k$ similar tokens from the sequence for each query token to compute the attention map. The proposed $k$-NN attention not only naturally inherits the local bias of CNNs as the nearby tokens tend to be more similar than others, but also builds the long range dependency by choosing the most similar tokens from the entire image. Compared with convolution operator which is an aggregation operation built on Ising model~\cite{paiconvolutional} and the feature of each node is aggregated from nearby pixels, in the $k$-NN attention, the aggregation graph is no longer limited by the spatial location of nodes but is adaptively computed via attention maps, thus, the $k$-NN attention can be regarded as a relieved version of local bias. The similar idea is proposed in~\cite{zhao2019explicit} where the $k$-NN attention is mostly evaluated on NLP tasks. Despite the similarity in terms of the calculation of top-$k$, our work focuses on the recent vision transformers, makes a deep theoretical understanding and presents a thoroughly analysis by defining several metrics.  We verify, both theoretically and empirically, that $k$-NN attention is effective in speeding up training and distilling noisy tokens of vision transformers. Eleven different available vision transformer architectures are adopted to verify the effectiveness of the proposed $k$-NN attention.

\section{Related Work}
\subsection{Self-attention}
Self-attention~\cite{vaswani2017attention} has demonstrated promising results on NLP related tasks, and is making breakthroughs in speech and computer vision. For time series modeling, self-attention operates over sequences in a step-wise manner. Specifically, at every time-step, self-attention assigns an attention weight to each previous input element and uses these weights to compute the representation of the current time-step as a weighted sum of the past inputs. Besides the vanilla self-attention, many efficient transformers~\cite{tay2020efficient} have been proposed.  Among these efficient transformers, sparse attention and local attention are one of the main streams, which are highly related to our work. Sparse attention can be further categorized into  data independent (fixed) sparse attention~\cite{child2019generating,ho2019axial,beltagy2020longformer,zaheer2020big} and content-based sparse attention~\cite{correia2019adaptively,roy2021efficient,kitaev2020reformer,tay2020sparse}. 
Local attention~\cite{luong2015effective,liu2018generating,liu2021swin} mainly considers attending only to a local window size. Our work is also content-based attention, but compared with previous works~\cite{correia2019adaptively,roy2021efficient,kitaev2020reformer,tay2020sparse}, our $k$-NN attention has its merits for vision domain. For example, compared with routing transformer~\cite{roy2021efficient} that clusters both queries and keys, our $k$-NN attention equals only clustering keys by assigning each query as the cluster center, making the quantization more continuous which is a better fitting of image domain; compared with reformer~\cite{kitaev2020reformer} which adopts complex hashing attention that cannot guarantee each bucket contain both queries and keys, our $k$-NN attention can guarantee that each query has number $k$ keys for attention computing. In addition, our $k$-NN attention is also a generalized local attention, but compared with local attention, our $k$-NN attention not only enjoys the locality but also empowers the ability of global relation mining.

\subsection{Transformer for Vision}
Transformer~\cite{vaswani2017attention} is an effective sequence-to-sequence modeling network, and it has achieved state-of-the-art results in NLP tasks with the success of BERT~\cite{devlin2018bert}.  Due to its great success, it has also be exploited in computer vision community, and `Transformer in CNN' becomes a popular paradigm~\cite{carion2020end,wang2020end,chen2021transformer,zou2021end,li2021transformer,li2021trear,han2020exploiting,chang2021augmented}. ViT~\cite{dosovitskiy2020image} leads the other trend to use `CNN in Transformer' paradigm for vision tasks~\cite{he2021transreid,li2021lifting,yu2021transrppg,xu2021cdtrans,yin2021transfgu}. Even though ViT has been proved compelling in vision recognition, it has several drawbacks when compared with CNNs: large training data, fixed position embedding, rigid patch division, coarse modeling of inner patch feature, single scale, unstable training process, slow speed training, easily fitting data and poor generalization, shallow \& narrow architecture, and quadratic complexity. To deal with these problems, many variants have been proposed~\cite{zhang2021aggregating,jonnalagadda2021foveater,wang2021not,fang2021msg,huang2021shuffle,gao2021container,rao2021dynamicvit,yu2021glance,zhou2021refiner,el2021xcit,wang2021crossformer,xu2021evo,zhou2021elsa}. For example, DeiT~\cite{touvron2020training} adopts several training techniques and uses distillation to extend ViT to a data-efficient version; CPVT~\cite{chu2021we} proposes a conditional positional encoding that is  adaptable to arbitrary input sizes; CvT~\cite{wu2021cvt}, CoaT~\cite{xu2021co} and Visformer~\cite{Zhengsu21} safely remove the position embedding by introducing convolution operations; T2T ViT~\cite{yuan2021tokens}, CeiT~\cite{yuan2021incorporating}, and CvT~\cite{wu2021cvt} try to deal with the rigid patch division by introducing convolution operation for patch sequence generation; Focal Transformer~\cite{yang2021focal} makes each token
attend its closest surrounding tokens at fine granularity and the tokens far away
at coarse granularity; TNT~\cite{han2021transformer} proposes the pixel embedding to model the inner patch feature; PVT~\cite{wang2021pyramid}, Swin Transformer~\cite{liu2021swin}, MViT~\cite{fan2021multiscale}, ViL~\cite{zhang2021multi}, CvT~\cite{wu2021cvt}, PiT~\cite{heo2021rethinking}, LeViT~\cite{graham2021levit}, CoaT~\cite{xu2021co}, and Twins~\cite{chu2021Twins} adopt multi-scale technique for rich feature learning; DeepViT~\cite{zhou2021deepvit}, CaiT~\cite{touvron2021going}, and PatchViT~\cite{gong2021improve} investigate the unstable training problem, and propose the re-attention, re-scale and anti-over-smoothing techniques respectively for stable training; to accelerate the convergence of training, ConViT~\cite{d2021convit}, PiT~\cite{heo2021rethinking}, CeiT~\cite{yuan2021incorporating}, LocalViT~\cite{li2021localvit} and Visformer~\cite{Zhengsu21} introduce convolutional bias to speedup the training; \textit{conv-stem} is adopted in LeViT~\cite{graham2021levit}, EarlyConv~\cite{xiao2021early}, CMT~\cite{guo2021cmt}, VOLO~\cite{yuan2021volo} and ScaledReLU~\cite{wang2021scaled} to improve the robustness of training ViTs; LV-ViT~\cite{jiang2021token} adopts several techniques including MixToken and Token Labeling for better training and feature generation;  T2T ViT~\cite{yuan2021tokens}, DeepViT~\cite{zhou2021deepvit} and  CaiT~\cite{touvron2021going} try to train deeper vision transformer models; T2T ViT~\cite{yuan2021tokens}, ViL~\cite{zhang2021multi} and CoaT~\cite{xu2021co} adopt efficient transformers~\cite{tay2020efficient} to deal with the quadratic complexity; To further exploit the capacities of vision transformer, OmniNet~\cite{tay2021omninet}, CrossViT~\cite{chen2021crossvit} and So-ViT~\cite{xie2021so} propose the dense omnidirectional representations, coarse-fine-grained patch fusion and cross co-variance pooling of visual tokens, respectively.  However, all of these works adopt the fully-connected self-attention which will bring the noise or 
irrelevant tokens for computing and slow down the training of networks. In this paper, we propose an efficient sparse attention, called $k$-NN attention, for boosting vision transformers. The proposed $k$-NN attention not only inherits the local bias of CNNs but also achieves the ability of global feature exploitation. It can also speed up the training and achieve better performance.

\section{$k$-NN Attention}

\subsection{Vanilla Attention}
For any sequence of length $n$, the vanilla attention in the transformer is the dot product attention~\cite{vaswani2017attention}. Following the standard notation, the attention matrix $\bm{A} \in \mathbbm{R}^{n \times n}$ is defined as:
\begin{equation}
    \bm{A} = \textrm{softmax}\left(\frac{\bm{Q}\bm{K}^{\top}}{\sqrt{d}}\right),\notag
\end{equation}
where $\bm{Q} \in \mathbbm{R}^{n \times d}$ denotes the queries while $\bm{K} \in \mathbbm{R}^{n \times d}$ denotes the keys, and $d$ represents the dimension. By multiplying the attention weights $\bm{A}$ with the values $\bm{V} \in \mathbbm{R}^{n \times d} $, the new values $\hat{\bm{V}}$ are calculated as:

\begin{equation}
    \hat{\bm{V}} = \bm{A}\bm{V}.\notag
\end{equation}

%
%
%
%
The intuitive understanding of the attention is the weighted average over the old ones, where the weights are defined by the attention matrix $\bm{A}$. In this paper, we consider the $\bm{Q}$, $\bm{K}$ and $\bm{V}$ are generated via the linear projection of the input token matrix $\bm{X}$: 

$$\bm{Q} = \bm{X}\bm{W}_{\bm{Q}},\ \bm{K} = \bm{X}\bm{W}_{\bm{K}},\ \bm{V} = \bm{X}\bm{W}_V,$$
where $\bm{X}\in \mathbbm{R}^{n\times d_m}$, $\bm{W}_{\bm{Q}},\bm{W}_{\bm{K}},\bm{W}_V\in\mathbbm{R}^{d_m\times d}$ and $d_m$ is the input token dimension.


One shortcoming with fully-connected self-attention is that irrelevant tokens, even though assigned with smaller weights, are still taken into consideration when updating the representation $\bm{V}$, making it less resilient to noises in $\bm{V}$. This shortcoming motivates us to develop the $k$-NN attention.

\subsection{$k$-NN Attention}
Instead of computing the attention matrix for all the query-key pairs as in vanilla attention, we select the top-$k$ most similar keys and values for each query in the $k$-NN attention. There are two versions of $k$-NN attention, as described below.

\textbf{Slow Version:} 
 For the $i$-th query, we first compute the Euclidean distance against all the keys, and then obtain its $k$-nearest neighbors $\mathcal{N}_{i}^{k}$ and $\mathcal{N}_{i}^{v}$ from keys and values
 , and lastly calculate the scaled dot product attention as: 

$$
   \bm{A}_{i} = \textrm{softmax}\left(\frac{\langle\bm{q}_{i} ,(\bm{k}_{j_1},...,\bm{k}_{j_l},...,\bm{k}_{j_k})\rangle}{\sqrt{d}}\right), \bm{k}_{j_l} \in \mathcal{N}_{i}^{k}.\notag
$$
The shape of final attention matrix is $\bm{A}^{knn}\in {\mathbbm{R}^{n \times k}}$, and the new values $\hat{\bm{V}}^{knn}$ is the same size of values $\hat{\bm{V}}$. The slow version is the exact definition of $k$-NN attention, but it is extremely slow because for each query it has to compute distances for different $k$ keys. 

\textbf{Fast Version:}
As the computation of Euclidean distance against all the keys for each query is slow, we propose a fast version of $k$-NN attention. The key idea is to take advantage of matrix multiplication operations. Same as vanilla attention, all the queries and keys are calculated by the dot product, and then row-wise top-$k$ elements are selected for $\mathrm{softmax}$ computing. The procedure can be formulated as:

$$
\hat{\bm{V}}^{knn}=  \textrm{softmax}\left(\mathcal{T}_k\left(\frac{\bm{Q}\bm{K}^{\top}}{\sqrt{d}}\right)\right)\cdot \bm{V},	\notag
$$
where $\mathcal{T}_{k}\left(\cdot\right)$ denotes the row-wise top-k selection operator:

$$
\left[\mathcal{T}_{k}(\bm{A})\right]_{ij} = \begin{cases}
 	A_{ij} \quad&{A_{ij}\in \textrm{top-k}(\textrm{row $j$}) }\\
 	-\infty \quad&\textrm{ otherwise.} 
 \end{cases}\notag
$$

\subsection{Theoretical Analysis on $k$-NN Attention}

In this section, we will show theoretically that despite its simplicity, $k$-NN attention is powerful in speeding up network training and in distilling noisy tokens.  All the proof of the lemmas are provided in the supplementary.

{\bf Convergence Speed-up.} Compared to CNNs, the fully-connected self-attention is able to capture long range dependency. However, the price to pay is that the dense self-attention model requires to mix each image patch with every other patch in the image, which has potential to mix irrelevant information together, e.g. the foreground patches may be mixed with background patches through the self-attention. This defect could significantly slow down the convergence as the goal of visual object recognition is to identify key visual patches relevant to a given class. 

To see this, we consider the model with only learnable parameters $\bm{W}_{\bm{Q}}$, $\bm{W}_{\bm{K}}$ in attention layers and adopting Adam optimizer \cite{kingma2014adam}. According to Theorem 4.1 in \cite{kingma2014adam}, Adam's convergence is proportional to $\mathcal{O}\left(\alpha^{-1}(G_{\infty}+1)+\alpha G_{\infty}\right)$, where $\alpha$ is the learning rate and $G_{\infty}$ is an element-wise upper bound on the magnitude of the batch gradient\footnote{ Theorem 4.1 in \cite{kingma2014adam} describes the upper bound for regrets (the gap  on loss function value between the current step parameters and optimal parameters). One can telescope it to the average regrets to consider the Adam's convergence.}. Let $f_i$ be the loss function corresponding to batch $i$. Via chain rule of derivative, the gradient w.r.t the $\bm{W}_{Q}$ in a self-attention block can be represented as $\nabla_{\bm{W}_{\bm{Q}}}f_i =F_i(\hat{\bm{V}}^{knn})\cdot \frac{\partial \hat{\bm{V}}^{knn}}{\partial \bm{W}_{\bm{Q}}}$, where $F_i(\hat{\bm{V}}^{knn})$ is a matrix output function. Since the possible value of $\hat{\bm{V}}^{knn}$ is a subset of its fully-connected counterpart, the upper bound of on the magnitude of $F_i(\hat{\bm{V}}^{knn})$ is no larger than the full attention. We then introduce the weighed covariate matrix of patches to characterize the scale of $\frac{\partial \hat{\bm{V}}^{knn}}{\partial \bm{W}_{\bm{Q}}}$ in the following lemma. 

\begin{lemma}\label{lem:1} (Informal)
Let $\hat{\bm{V}}^{knn}_l$ be the $l$-th row of the $\hat{\bm{V}}^{knn}$.
We have

$$
\frac{\partial \hat{\bm{V}}_l^{knn}}{\partial \bm{W}_{Q}}\propto \textrm{Var}_{\bm{a}_l}(\bm{x})\textrm{ and } \frac{\partial \hat{\bm{V}}_l^{knn}}{\partial \bm{W}_{K}} \propto \textrm{Var}_{\bm{a}_l}(\bm{x}),\notag
$$
where $\textrm{Var}_{\bm{a}_l}(\bm{x})$ is the covariate matrix on patches $\{\bm{x}_1,...,\bm{x}_n\}$ with probability from $l$-th row of the attention matrix.\\ The same is true for $\hat{\bm{V}}$ of the fully-connected self-attention.
%
\end{lemma}
Since $k$-NN attention only uses patches with large similarity, its $\textrm{Var}_{\bm{a}_l}(\bm{x})$ will be smaller than that computed from the fully-connected attention. As indicated in Lemma \ref{lem:1},  $\frac{\partial \hat{\bm{V}}^{knn}}{\partial \bm{W}_{\bm{Q}}}$ is proportional to variance $\textrm{Var}_{\bm{a}_l}(\bm{x})$ and thus the scale of $\nabla_{\bm{W}_{\bm{Q}}} f_i$ becomes smaller in k-NN attention. Similarly, the scale of $\nabla_{\bm{W}_{\bm{K}}} f_i$ is also smaller in k-NN attention. Therefore, the element-wise upper bound on batch gradient $G_{\infty}$ in Adam analysis is also smaller for k-NN attention.  For the same learning rate, the k-NN attention yields faster convergence.  
 It is particularly significant at the beginning of training. This is because, due to the random initialization, we expect a relatively small difference in similarities between patches, which essentially makes self-attention behave like ``global average". It will take multiple iterations for Adam to turn the "global average" into the real function of self-attention. In Table~\ref{trainingspeed} and Figure~\ref{properties}, we numerically verify the training efficiency of $k$-NN attention as opposed to the fully-connected attention.

{\bf Noisy patch distillation.}  As already mentioned before, the fully-connected self-attention model may mix irrelevant patches with relevant ones, particularly at the beginning of training when similarities between relevant patches are not significantly larger than those for irrelevant patches. $k$-NN attention is more effective in identifying noisy patches by only considering the top $k$ most similar patches. To formally justify this point, we consider a simple scenario where all the patches are divided into two groups, the group of relevant patches and the group of noisy patches. All the patches are sampled independently from unknown distributions. We assume that all relevant patches are sampled from distributions with the same shared mean, which is different from the means of distributions for noisy patches. It is important to know that although distributions for the relevant patches share the mean, those relevant patches can look quite differently, due to the large variance in stochastic sampling. In the following Lemma, we will show that the $k$-NN attention is more effective in distilling noises for the relevant patches than the fully-connected attention.



\begin{lemma}[informal]\label{lem:3}
We consider the self-attention for query patch $l$. Let's assume the patch $\bm{x}_i$ are bounded with mean $\bm{\mu}_i$ for $i=1,2,...,n$, and $ \rho_k$ is the ratio of the noisy patches in all selected patches.  Under mild conditions, the follow inequality holds with high probability:

$$
\left\|\hat{\bm{V}}_l^{knn}-\bm{\mu}_l\bm{W}_V\right\|_{\infty}\le \mathcal{O} (k^{-1/2}+c_1\rho_k),\notag
$$
where $c_1$ is a  positive number.
\end{lemma}


In the above lemma, the quantity $\left\|\hat{\bm{V}}_l^{knn}-\bm{\mu}_l\bm{W}_V\right\|_{\infty}$ measures the distance between $\hat{\bm{V}}^{knn}_l$, represention vector updated by the $k$-NN attention, and its mean $\mu_l\bm{W}_{\bm{V}}$. We now consider two cases: the normal $k$-NN attention with appropriately chosen $k$, and fully-connected attention with $k=n$. In the first case, with appropriately chosen $k$, we should have most of the selected patches coming from the relevant group, implying a small $\rho_k$. By combining with the fact that $k$ is decently large, we expect a small upper bound for the distance $\left\|\hat{\bm{V}}_l^{knn}-\bm{\mu}_l\bm{W}_V\right\|_{\infty}$, indicating that $k$-NN attention is powerful in distilling noise. For the case of fully-connected attention model, i.e. $k=n$, it is clearly that $\rho_n \approx 1$, leading to a large distance between transformed representation $\hat{\bm{V}}_l$ and its mean, indicating that fully-connected attention model is not very effective in distilling noisy patches, particularly when noise is large.

Besides the instance with low signal-noise-ratio, the instance with a large volume of backgrounds can also be hard. In the next lemma, we show that under a proper choice of $k$, with a high probability the $k$-NN attention will be able to select all meaningful patches.

\begin{lemma}[informal]\label{lem:2}
	Let $\mathcal{M}^*$ be the index set contains all patches relevant to query $\bm{q}_l$.
	Under mild conditions, there exist $c_2\in(0,1)$ such that with high probability, we have

$$
		\sum_{i=1}^n\mathbbm{1}(\bm{q}_l\bm{k}_i^{\top}\ge \min_{j\in\mathcal{M}^*}\bm{q}_l\bm{k}_j^{\top})\le \mathcal{O}(nd^{-c_2}).\notag
$$
\end{lemma}
 
The above lemma shows that if we select the top $\mathcal{O}(nd^{-c_2})$ elements, with high probability, we will be able to eliminate almost all the irrelevant noisy patches, without losing any relevant patches. Numerically, we verify the proper $k$ gains better performance (e.g., Figure~\ref{impactofk}) and for the hard instance $k$-NN gives more accurate attention regions. (e.g., Figure~\ref{Visualization1} and Figure~\ref{Visualization2}).

\section{Experiments for Vision Transformers}
In this section, we replace the dense attention with $k$-NN attention on the existing vision transformers for image classification to verify the effectiveness of the proposed method. The recent DeiT~\cite{touvron2020training} and its variants, including T2T ViT~\cite{yuan2021tokens}, TNT~\cite{han2021transformer}, PiT~\cite{heo2021rethinking}, Swin~\cite{liu2021swin}, CvT~\cite{wu2021cvt}, So-ViT~\cite{xie2021so}, Visformer~\cite{Zhengsu21}, Twins~\cite{chu2021Twins},  Dino~\cite{caron2021emerging} and VOLO~\cite{yuan2021volo}, are adopted for evaluation. These methods include both supervised methods~\cite{touvron2020training,yuan2021tokens,han2021transformer,heo2021rethinking,liu2021swin,wu2021cvt,xie2021so,Zhengsu21,chu2021Twins,yuan2021volo} and self-supervised method~\cite{caron2021emerging}. Ablation studies are provided to further analyze the properties of $k$-NN attention.  

\subsection{Experimental Settings}
We perform image classification on the standard ILSVRC-2012 ImageNet dataset~\cite{russakovsky2015imagenet}. In our experiments, we follow the experimental setting of original official released codes. For fair comparison, we only replace the vanilla attention with proposed k-NN attention. Unless otherwise specified, the fast version of $k$-NN attention is adopted for evaluation. To speed up the slow version, we develop the CUDA version $k$-NN attention. As for the value $k$, different architectures are assigned with different values. For DeiT~\cite{touvron2020training}, So-ViT~\cite{xie2021so}, Dino~\cite{caron2021emerging}, CvT~\cite{wu2021cvt}, TNT~\cite{han2021transformer} PiT~\cite{heo2021rethinking} and VOLO~\cite{yuan2021volo}, as they directly split an input image into rigid tokens and there is no information exchange in the token generation stage, we suppose the irrelevant tokens are easy to filter, and tend to assign a smaller $k$ compared with these complicated token generation methods~\cite{yuan2021tokens,liu2021swin,Zhengsu21,chu2021Twins}.  Specifically, we assign $k$ to approximate $\frac{n}{2}$ at each scale stage; for these complicated token generation methods~\cite{yuan2021tokens,liu2021swin,Zhengsu21,chu2021Twins}, we assign a larger $k$ which is approximately $\frac{2}{3}{n}$ or $\frac{4}{5}{n}$ at each scale stage.

\subsection{Results on ImageNet}
Table~\ref{imagenet} shows top-$1$ accuracy results on the ImageNet-1K validation set by replacing the dense attention with $k$-NN attention using eleven different vision transformer architectures. From the Table we can see that the proposed $k$-NN attention improves the performance from 0.2\% to 0.8\% for both global and local vision transformers. It is worth noting that on ImageNet-1k dataset, it is very hard to improve the accuracy after 85\%, but our $k$-NN attention can still consistently improve the performance even without model size increase. 

\begin{table*}[ht!]
      \caption{The $k$-NN attention performance on ImageNet-1K validation set. "!" means we pretrain the model with 300 epochs and finetune the pretrained model for 100 epoch for linear eval, following the instructions of Dino training and evaluation;  "$\rightarrow$ $k$-NN Attn" represents replacing the vanilla attention with proposed $k$-NN attention;$\rightarrow$ $k$-NN Attn-slow means adopting the slow version. }\label{imagenet}
  \begin{center}
    \begin{tabular}{@{}l|l|l|l|l|l@{}}
      \toprule
      Arch.    & Model & Input &Params & GFLOPs  & Top-1 \\ \midrule
      Transformers  & DeiT-Tiny~\cite{touvron2020training} & $224^{2}$ & 5.7M & 1.3 &                  72.2\% \\ 
      (Supervised)             & DeiT-Tiny~\cite{touvron2020training} $\rightarrow$ k-NN Attn & $224^{2}$ & 5.7M & 1.3 & 73.0\% \\ 
                   & DeiT-Tiny~\cite{touvron2020training} $\rightarrow$ k-NN Attn-slow & $224^{2}$ & 5.7M & 1.3 & 73.0\% \\ 
                   & So-ViT-7~\cite{xie2021so} & $224^{2}$ & 5.5M & 1.3 & 76.2\% \\
                   & So-ViT-7~\cite{xie2021so} $\rightarrow$ k-NN Attn & $224^{2}$ & 5.5M & 1.3 & 77.0\% \\
                    
     \bottomrule
     Transformers  & Visformer-Tiny~\cite{Zhengsu21} & $224^{2}$ & 10M & 1.3 & 78.6\% \\ 
     (Supervised)               & Visformer-Tiny~\cite{Zhengsu21} $\rightarrow$ k-NN Attn & $224^{2}$ & 10M & 1.3 & 79.0\% \\ 
      \bottomrule
     Transformers & CvT-13~\cite{wu2021cvt}  & $224^{2}$ & 20M & 4.6 & 81.6\% \\
      (Supervised)       & CvT-13~\cite{wu2021cvt} $\rightarrow$ k-NN Attn & $224^{2}$ & 20M & 4.6 & 81.9\% \\
            & DeiT-Small~\cite{touvron2020training} & $224^{2}$ & 22M & 4.6 & 79.8\% \\ 
            & DeiT-Small~\cite{touvron2020training} $\rightarrow$ k-NN Attn & $224^{2}$ & 22M & 4.6 & 80.1\% \\  
            & TNT-Small~\cite{han2021transformer} & $224^{2}$ & 24M & 5.2 & 81.5\% \\ 
            & TNT-Small~\cite{han2021transformer} $\rightarrow$ k-NN Attn & $224^{2}$ & 24M & 5.2 & 81.9\% \\ 
            & VOLO-D1~\cite{yuan2021volo}  & $384^{2}$ & 27M & 22.8 & 85.2\% \\
            & VOLO-D1~\cite{yuan2021volo} $\rightarrow$ k-NN Attn & $384^{2}$ & 27M & 22.8 & 85.4\% \\
            &Swin-Tiny~\cite{liu2021swin}  & $224^{2}$ & 28M & 4.5 & 81.2\% \\
            &Swin-Tiny~\cite{liu2021swin} $\rightarrow$ k-NN Attn & $224^{2}$ & 28M & 4.5 & 81.3\% \\
            & T2T-ViT-t-19~\cite{yuan2021tokens}  & $224^{2}$ & 39M & 9.8 & 82.2\% \\ 
            & T2T-ViT-t-19~\cite{yuan2021tokens} $\rightarrow$ k-NN Attn & $224^{2}$ & 39M & 9.8 & 82.7\% \\  
    \midrule
    Transformer   & Dino-Small~\cite{caron2021emerging}! & $224^{2}$ & 22M & 4.6 & 76.0\%  \\
  (Self-supervised)   & Dino-Small~\cite{caron2021emerging}! $\rightarrow$ k-NN Attn & $224^{2}$ &                    22M & 4.6 & 76.2\% \\
    \midrule
    Transformers 
                 & Twins-SVT-Base~\cite{chu2021Twins}  & $224^{2}$ & 56M & 8.3 & 83.2\% \\
    (Supervised)              & Twins-SVT-Base~\cite{chu2021Twins} $\rightarrow$ k-NN Attn & $224^{2}$ & 56M & 8.3 & 83.4\% \\
                 & PiT-Base\cite{heo2021rethinking}  & $224^{2}$ & 74M & 12.5 & 82.0\% \\
                 & PiT-Base\cite{heo2021rethinking} $\rightarrow$ k-NN Attn & $224^{2}$ & 74M & 12.5 & 82.6\% \\
                 & VOLO-D3~\cite{yuan2021volo}  & $448^{2}$ & 86M & 67.9 & 86.3\% \\
                 & VOLO-D3~\cite{yuan2021volo} $\rightarrow$ k-NN Attn & $448^{2}$ & 86M & 67.9 & \textbf{86.5\%} \\
           \bottomrule
      \end{tabular}
  \end{center}  
\end{table*}

\subsection{The Impact of Number $k$}
The only parameter for $k$-NN attention is $k$, and its impact is analyzed in Figure~\ref{impactofk}. As shown in the figure, for DeiT-Tiny, $k$ = 100 is the best, where the total number of tokens $n$ = 196 (14 $\times$ 14), meaning that $k$ approximates half of $n$; for CvT-13, there are three scale stages with the number of tokens $n_{1}$ = 3136, $n_{2}$ = 784 and $n_{3}$ = 196, and the best results are achieved when the $k$ in each stage is assigned to 1600/400/100, which also approximate half of $n$ in each stage; for Visformer-Tiny, there are two scale stages with the number of tokens  $n_{1}$ = 196 and $n_{2}$ = 49, and the best results are achieved when $k$ in each stage is assigned to 150/45, as there are more than 21 conv layers for token generation and the information in each token are already mixed, making it hard to distinguish the irrelevant tokens, thus larger values of $k$ are desired; for PiT-Base, there are three scale stages with the number of tokens $n_{1}$ = 961, $n_{2}$ = 256 and $n_{3}$ = 64, and the optimal values of $k$  also  approximate the half  of $n$. Please note that, we do not  perform exhaustive search for the optimal choice of $k$, instead, a general rule as below is sufficient:  $k \approx $ $\frac{n}{2}$ at each scale stage for simple token generation methods and $k \approx \frac{2}{3}{n}$ or $\frac{4}{5}{n}$ for complicated token generation methods at each scale stage. 

\begin{figure}[!]
  \begin{minipage}{0.24\linewidth}
    \begin{center}
      \includegraphics[width=\linewidth]{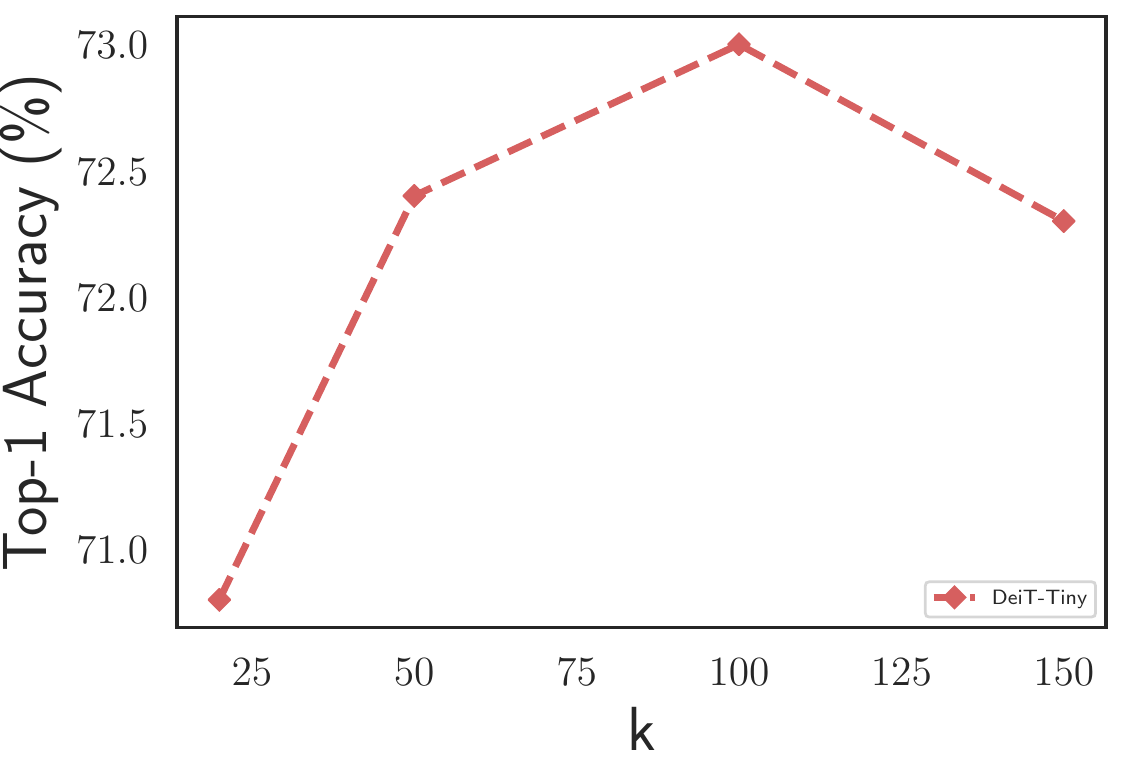}
      (a) DeiT-Tiny
    \end{center}
  \end{minipage}
    \begin{minipage}{0.24\linewidth}
    \begin{center}
      \includegraphics[width=\linewidth]{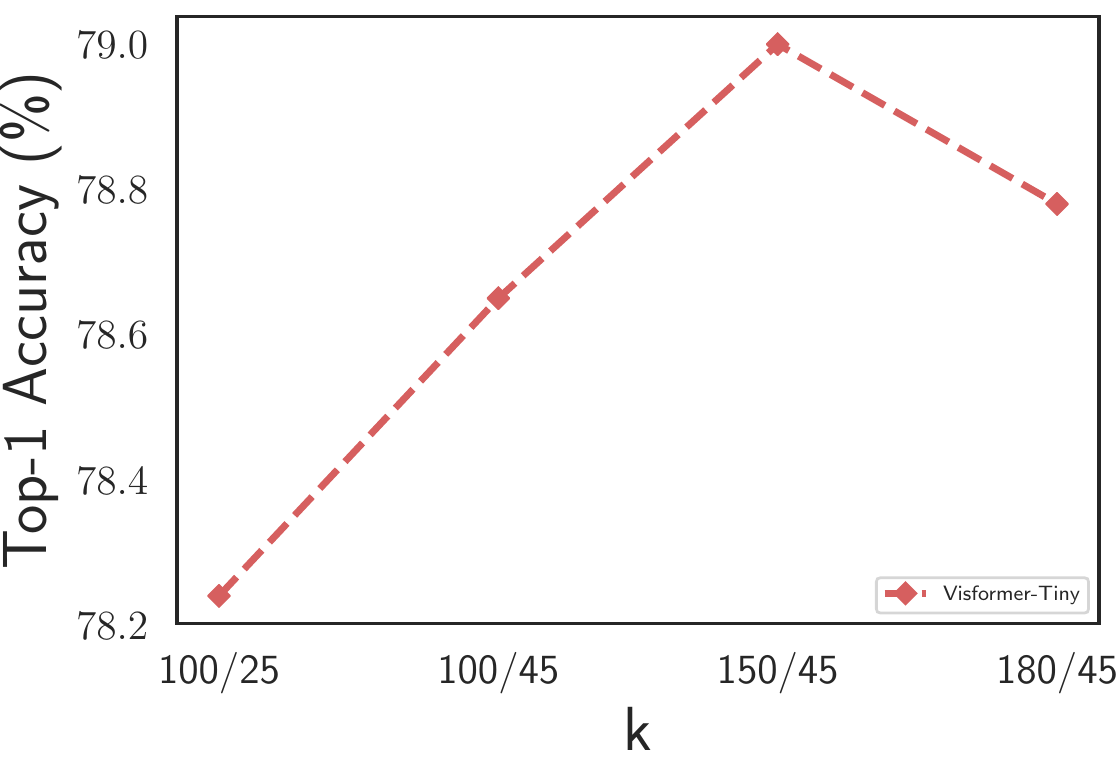}
      (c) Visformer-Tiny
    \end{center}
  \end{minipage}
    \begin{minipage}{0.24\linewidth}
    \begin{center}
      \includegraphics[width=\linewidth]{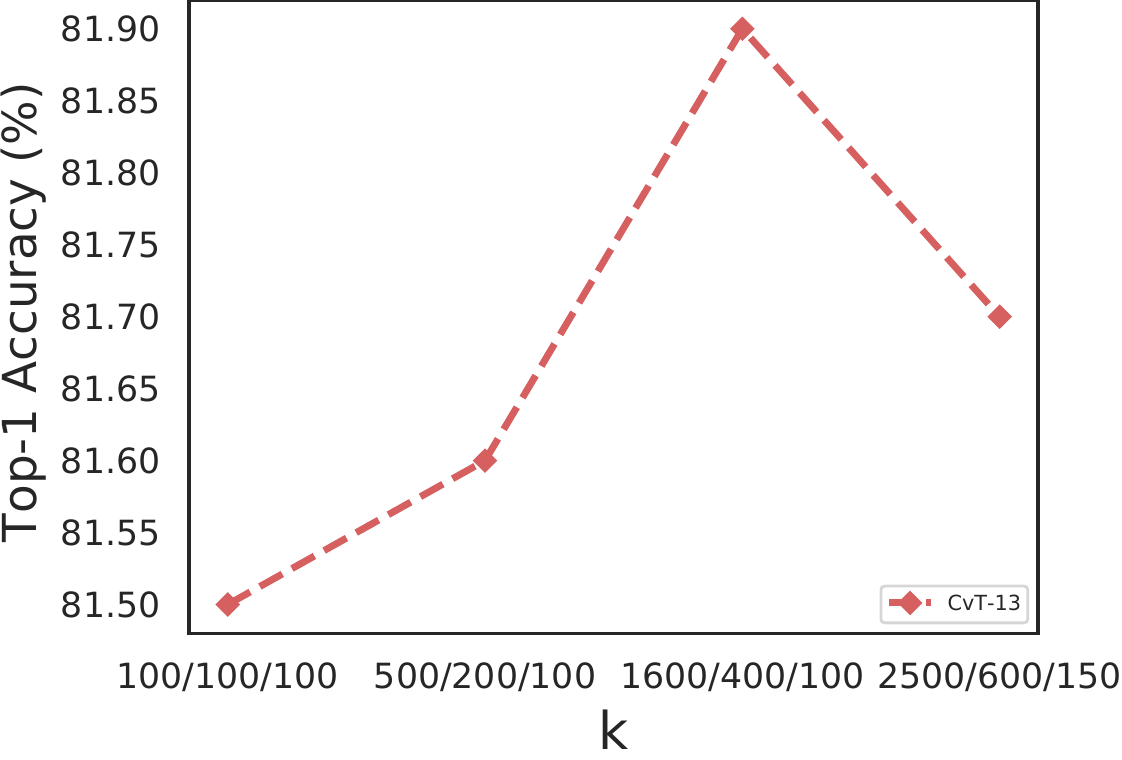}
      (b) CvT-13
    \end{center}
  \end{minipage}
      \begin{minipage}{0.24\linewidth}
    \begin{center}
      \includegraphics[width=\linewidth]{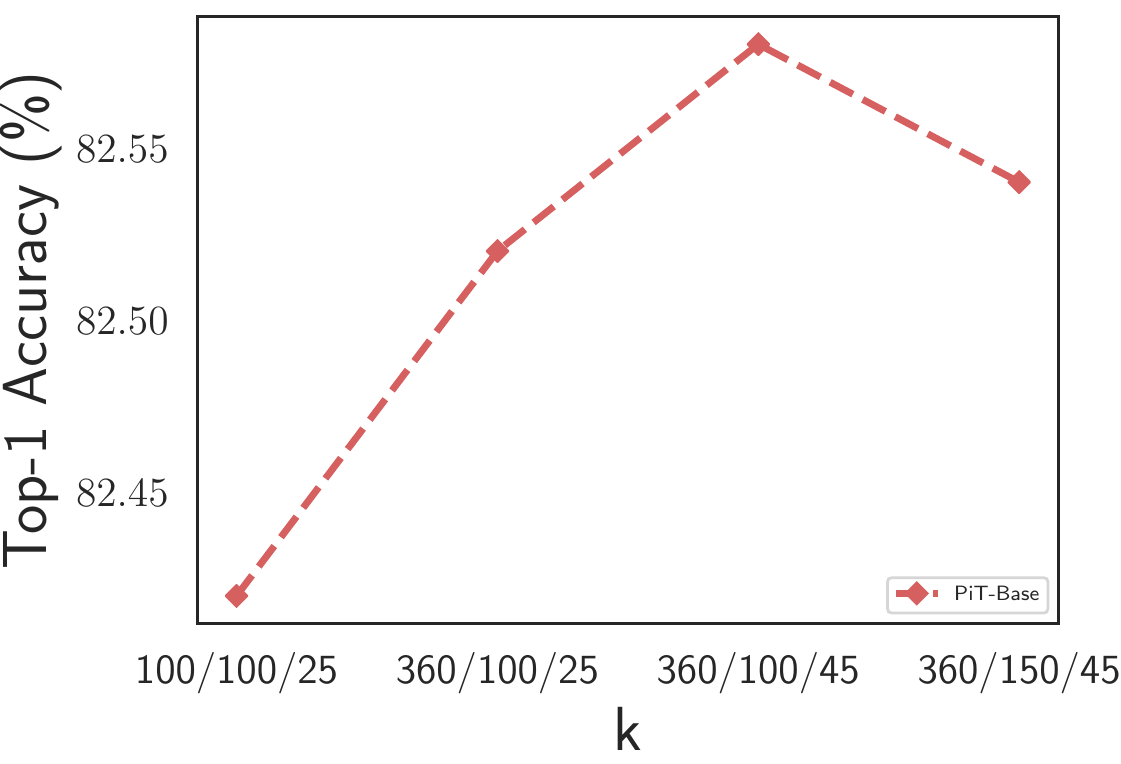}
      (d) PiT-Base
    \end{center}
  \end{minipage}
  \caption{The impact of $k$ on DeiT-Tiny, Visformer-Tiny, CvT-13 and PiT-Base.}
  \label{impactofk}
\end{figure}

\subsection{Convergence Speed of $k$-NN Attention}
In Table~\ref{trainingspeed}, we investigate the convergence speed of $k$-NN attention. Three methods are included for comparison, i.e. DeiT-Small~\cite{touvron2020training}, CvT-13~\cite{wu2021cvt} and T2T-ViT-t-19~\cite{yuan2021tokens}. From the Table we can see that the convergence speed of $k$-NN attention is faster than full-connected attention, especially in the early stage of training. These observations reflect that  removing the irrelevant tokens benefits the convergence of neural networks training.
\begin{table*}
\caption{Ablation study on the convergence speed of $k$-NN attention.}\label{trainingspeed}
\centering
\begin{tabular}{c|c|c|c|c|c|c} 
\toprule
\multirow{2}{*}{Epoch} & \multicolumn{6}{c}{Top-1 accuracy}                             \\ 
\cline{2-7}
& DeiT-S & DeiT-S $\rightarrow$ k & CvT-13 & CvT-13 $\rightarrow$ k & T2T-ViT-t-19 & T2T-ViT-t-19 $\rightarrow$ k \\ 
\midrule
10                     &  29.1\%    &         31.3\%            &  51.4\%   &   54.2\%   & 0.52\% &  0.68\%             \\ 

30                     &  54.4\%    &         55.4\%            &  65.4\%   &   68.1\%   &   63.0\% &  63.2\%   \\ 

50                     &   60.9\%   &         62.0\%            &  68.1\%   &   70.5\%   &   73.8\%  &  74.4\%       \\ 

70                     &   65.0\%   &         65.8\%            &  69.9\%  &    72.2\%   &   76.9\%  &  77.3\%     \\ 

90                     &    67.7\%  &         68.2\%            &  71.0\%   &   73.0\%   &   78.4\%  &  78.6\%     \\ 
120                     &    69.9\%  &        70.7\%            &  72.4\%   &   73.7\%   &   79.7\%  &  80.0\%      \\ 

150                    &    72.4\%  &         72.4\%            &   74.4\%  &   74.9\%   &   80.7\% &   80.9\%          \\ 

200                    &   75.5\%   &         75.7\%            &   77.3\%  &   77.7\%   & 82.0\% &    82.3\%        \\ 

300                    &    79.8\%  &         80.0\%            &   81.6\%  &   81.9\%  &  81.3\%   & 81.7\% \\
\bottomrule
\end{tabular}
\end{table*}


\subsection{Other properties of $k$-NN attention}
 To analyze other properties of $k$-NN attention, four quantitative metrics are defined as follows.\\
\textbf{Layer-wise cosine similarity between tokens:} following~\cite{gong2021improve} this metric is defined as:

$$
    \textrm{CosSim(\textbf{t})} = \frac{1}{n(n-1)}\sum_{i \neq j}\frac{t^{T}_{i}t_{j}}{\lVert t_{i}\rVert\lVert t_{j}\rVert}\notag,
$$
where $t_{i}$ represents the $i$-th token in each layer and $\lVert \cdot \rVert$ denotes the Euclidean norm. This metric implies the convergence speed of the network. 

\textbf{Layer-wise standard deviation of attention weights:} Given a token $t_{i}$ and its \textrm{softmax} attention weight \textrm{sfm($t_{i}$)}, the standard deviation of the \textrm{softmax} attention weight \textrm{std}(\textrm{sfm($t_{i}$)}) is defined as the second metric. For multi-head attention, the standard deviations over all heads are averaged. This metric represents the degree of training stability. 

\textbf{Ratio between the norms of residual activations and main branch:} The ratio between the norm of the residual activations and the norm of the activations of the main branch in each layer is defined as $\lVert f_{l}(t) \rVert / \lVert t \rVert$, where $f_{l}(t)$ can be the attention layer or the FFN layer. This metric denotes the information preservation ability of the network.

\textbf{Nonlocality:} following~\cite{d2021convit}, the nonlocality is defined by summing, for each query patch $i$, the distances $\left\|\delta_{ij}\right\|$ to all the key patches $j$ weighted by their attention score $\bm{A}_{ij}$. The number obtained over the query patch is averaged to obtain the nonlocality metric of head $h$, which can the be averaged over the attention heads to obtain the nonlocality of the whole layer $l$:
$$
    D_{loc}^{l,h}:= \frac{1}{L}\sum_{ij}\bm{A}_{ij}^{h,l}\left\|\delta_{ij}\right\|,
    D_{loc}^{l}:= \frac{1}{N_{h}}\sum_{h}D_{loc}^{l,h}\notag,
$$
where $D_{loc}$ is the number of patches between the center of attention and the query patch; the further the attention heads look from the query patch, the higher the nonlocality. 

\begin{figure}[ht!]
  \begin{minipage}{0.48\linewidth}
    \begin{center}
      \includegraphics[width=\linewidth]{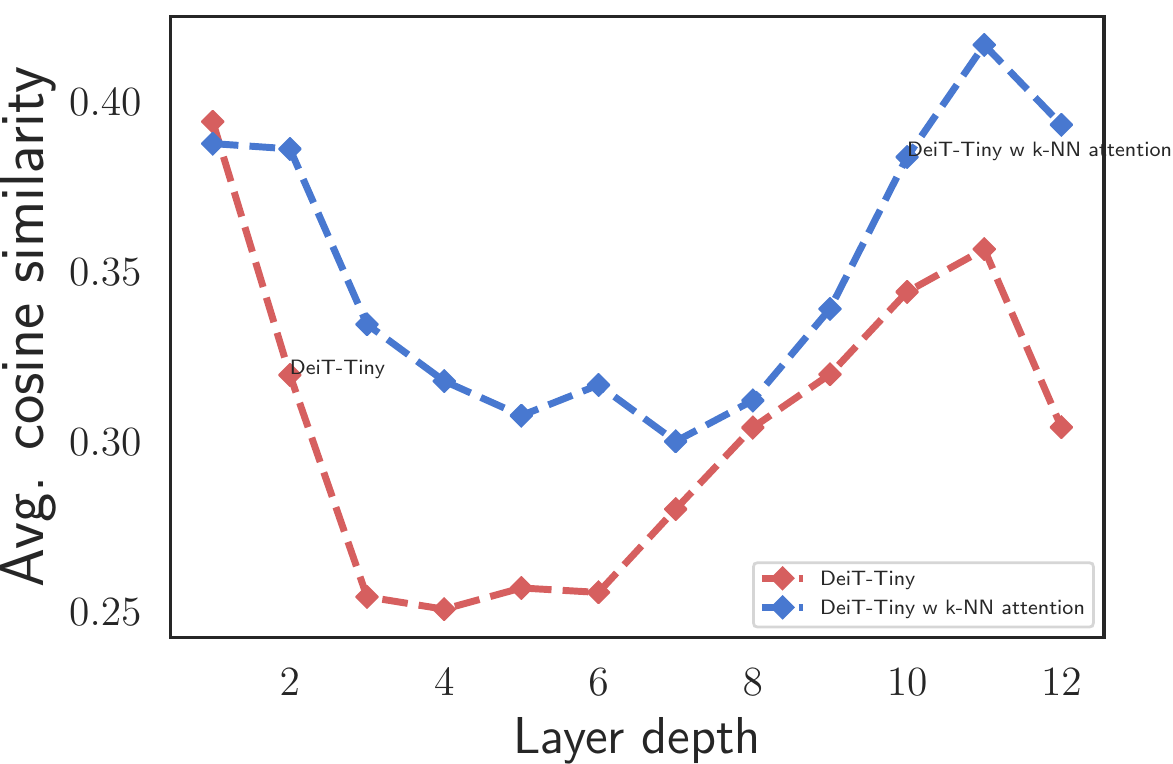}
      (a) Layer-wise cosine similarity of tokens
    \end{center}
  \end{minipage}
  \begin{minipage}{0.48\linewidth}
    \begin{center}
      \includegraphics[width=\linewidth]{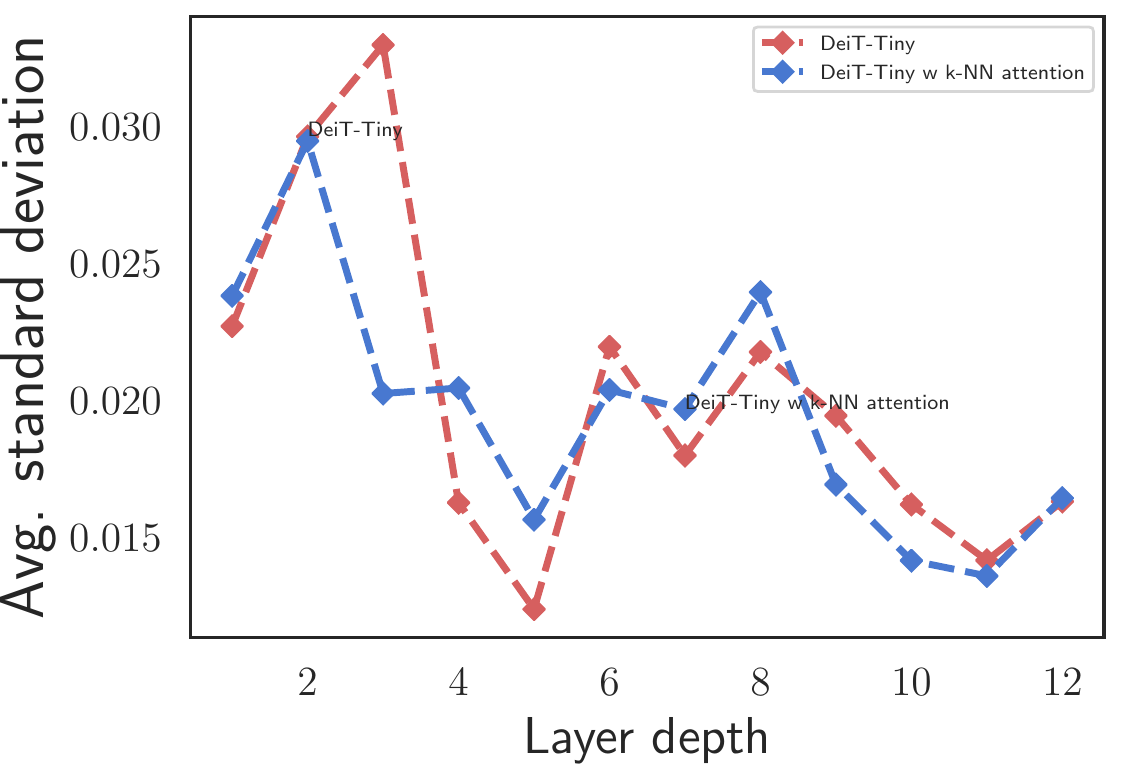}
      (b) Layer-wise $s.t.d$ of attention weights
    \end{center}
  \end{minipage}
    \begin{minipage}{0.48\linewidth}
    \begin{center}
      \includegraphics[width=\linewidth]{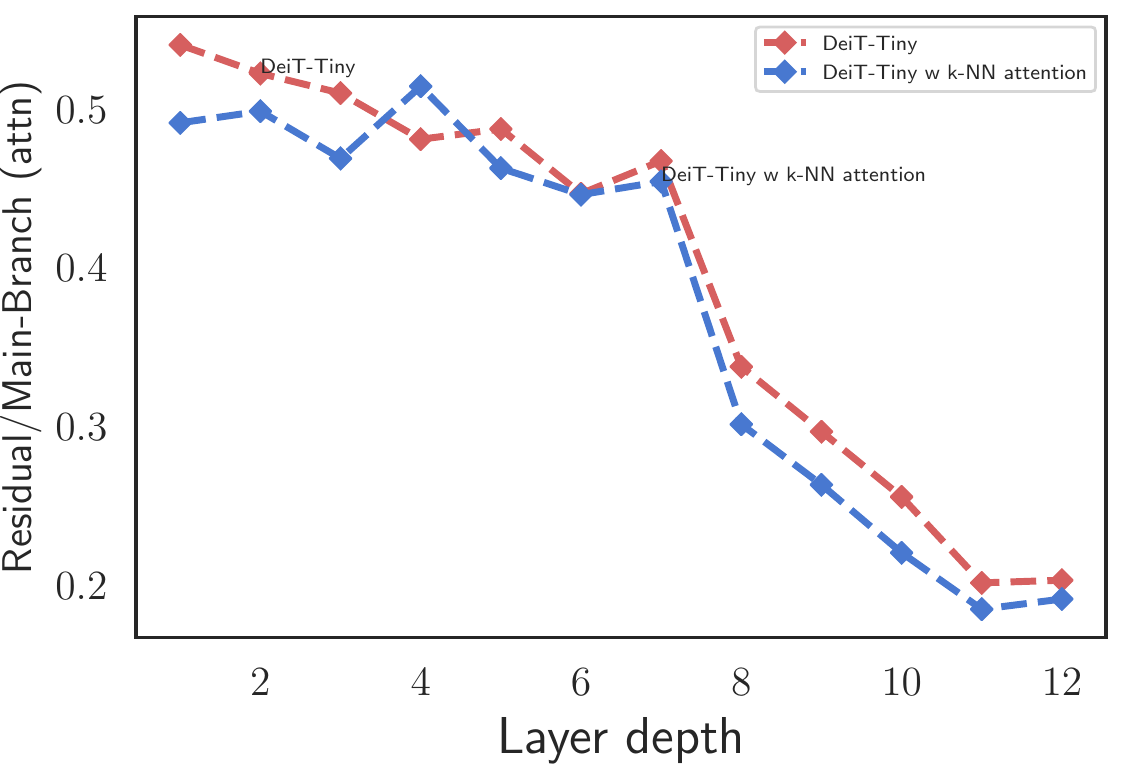}
      (c) Ratio of residual and main branch for attn
    \end{center}
  \end{minipage}
      \begin{minipage}{0.48\linewidth}
    \begin{center}
      \includegraphics[width=\linewidth]{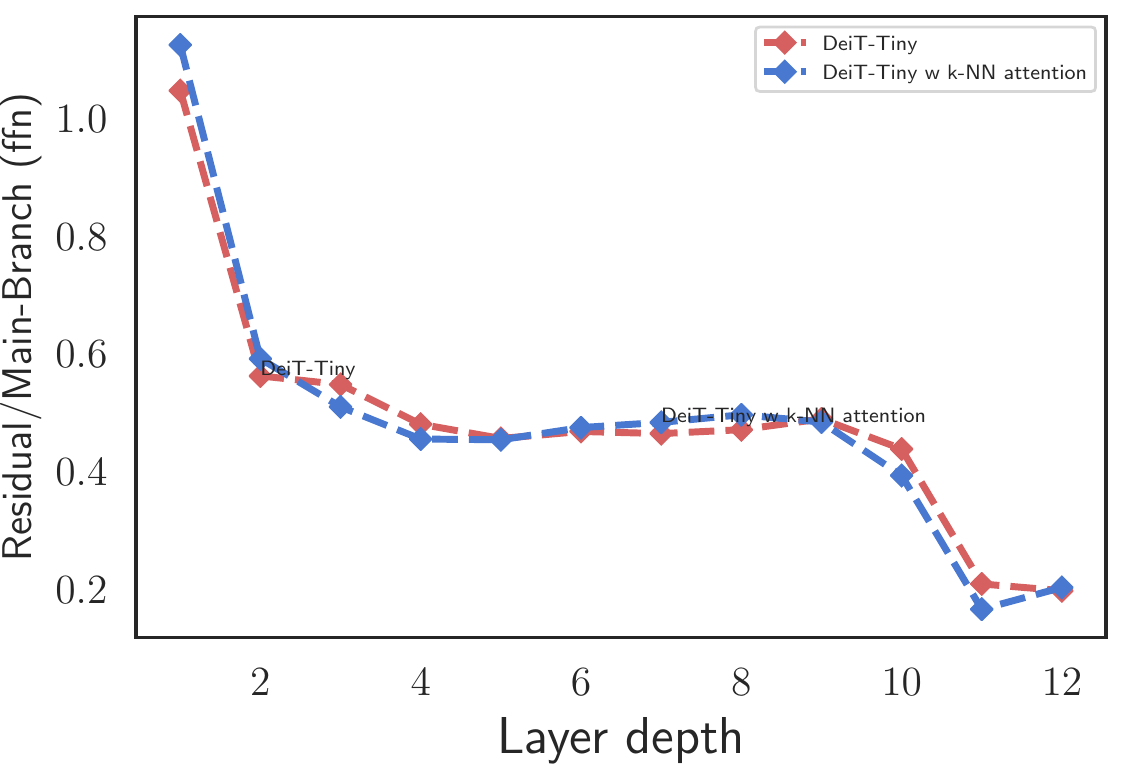}
      (d) Ratio of residual and main branch for ffn
    \end{center}
  \end{minipage}
  \caption{The properties of $k$-NN attention. Blue and red dotted lines represent the metrics for $k$-NN attetion and the original fully-connected self-attention, respectively.}\label{properties}
\end{figure}

Comparisons of the four metrics on DeiT-tiny without distillation token are shown in Figure~\ref{properties} and Figure~\ref{nonlocality}. From Figure~\ref{properties} (a) we can see that by using $k$-NN attention, the averaged cosine similarity is larger than that of using dense self-attention, which reflects that the convergence speed is faster for $k$-NN attention.  Figure~\ref{properties} (b) shows that  the averaged standard deviation of $k$-NN attention is smoother than that of fully-connected self-attention, and the smoothness will help make the training more stable. Figure~\ref{properties} (c) and (d) show the ratio between the norms of residual activations and main branch are consistent with each other for $k$-NN attention and dense attention, which indicates that there is nearly no information lost in $k$-NN attention by removing the irrelevant tokens. Figure~\ref{nonlocality} shows that, with k-NN attention, lower layers tend to focus more on the local areas (with more lines being pushed toward the bottom area in Figure 3), while the higher layers still maintain their capability of extracting global information. Additionally, it is also observed that the non-locality of different layers is spreading more evenly, indicating that they can explore a larger variety of dependencies at different ranges.

\begin{figure}[ht!]
 \begin{minipage}{0.4\linewidth}
    \begin{center}
      \includegraphics[width=\linewidth,height=\linewidth]{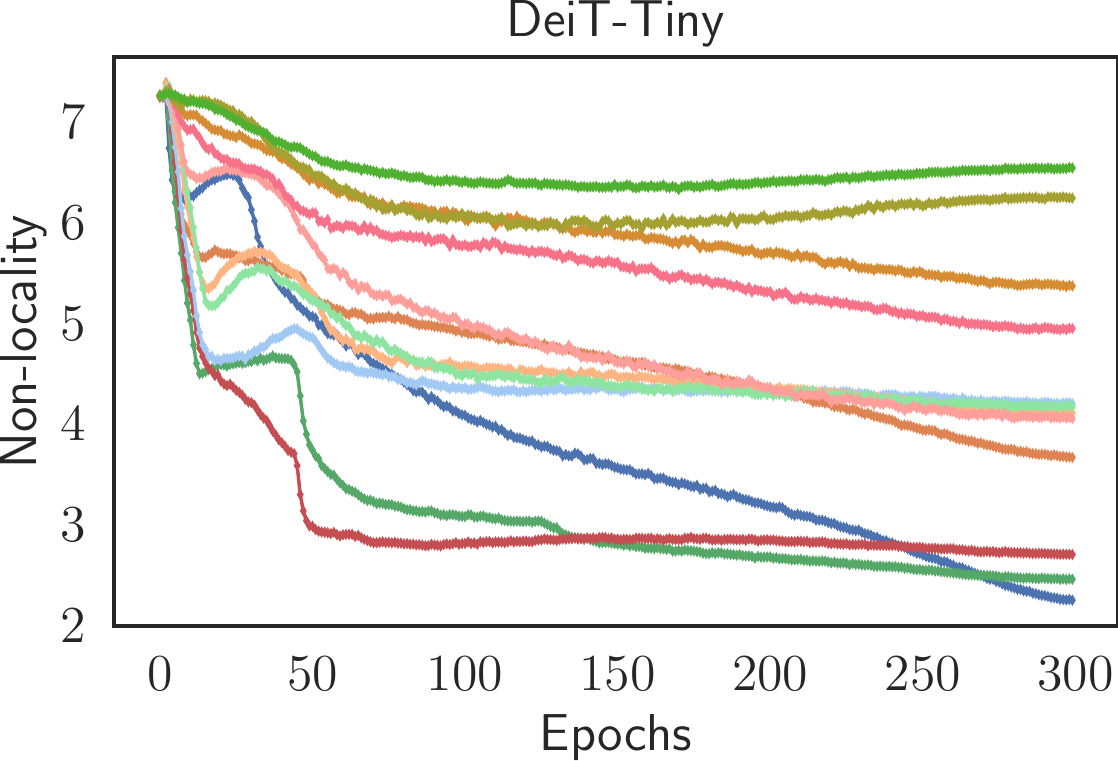}
    \end{center}
  \end{minipage}
      \begin{minipage}{0.59\linewidth}
    \begin{center}
      \includegraphics[width=\linewidth,height=0.68\linewidth]{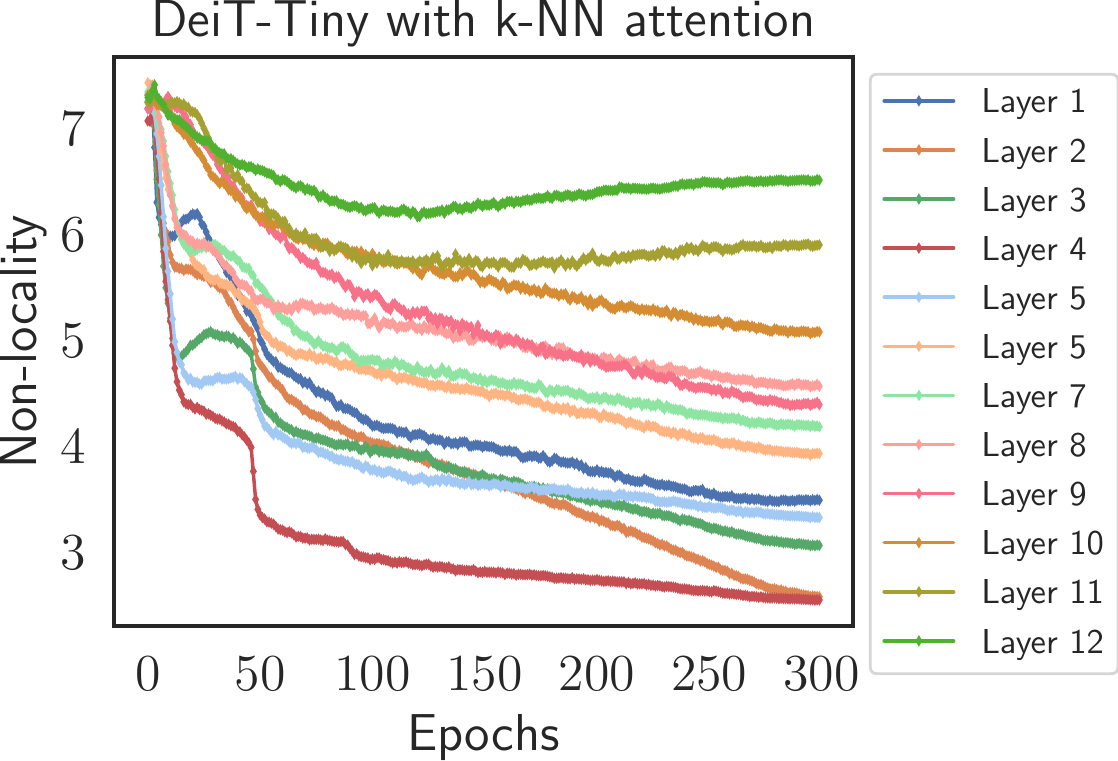}
    \end{center}
  \end{minipage}
  \caption{The nonlocality of DeiT-Tiny. It is plotted averaged over all the images from training set of ImageNet-1k.}\label{nonlocality}
\end{figure}

\subsection{Comparisons with temperature in softmax}
$k$-NN attention effectively zeros the bottom $N-k$ tokens out of the attention calculation. How does this compare with introducing a temperature parameter to \textrm{softmax} over the attention values? We compare our $k$-NN attention with temperature $t$ in \textrm{softmax} as \textrm{softmax(attn/$t$)}. The performance over the $t$ is shown in Table~\ref{temperature}. From the Table we can see that small $t$ makes the training crash due to large value of attention values; the performance increases a little bit to 72.5 (baseline 72.2) with $t$ assigned to appropriate values. The $k$-NN attention is more robust compared with temperature in \textrm{softmax}, and achieves much better performance, 73.0 ($k$-NN attention) vs 72.5 (best performance for temperature in \textrm{softmax}).

\begin{table}
\caption{The Top-$1$ (\%) over the temperature $t$ in \textrm{softmax}.}\label{temperature}
\centering
\begin{tabular}{c|c|c|c|c|c|c|c|c}
\hline
$t$ & 0.05 & 0.1 & 0.25 & 0.75  & 2 & 4 & 8 & 16 \\
\hline
Top-$1$ (\%) & crash & crash & 72.0 & 72.5 & 72.5 & 72.5 & 72.5 & 72.1  \\
\hline
\end{tabular}
\end{table}

\begin{figure}[ht!]
\begin{minipage}{0.02\linewidth}
    \begin{center}
     \begin{sideways} dense \end{sideways} 
    \end{center}
  \end{minipage}
  \begin{minipage}{0.131\linewidth}
    \begin{center}
      \includegraphics[width=\linewidth]{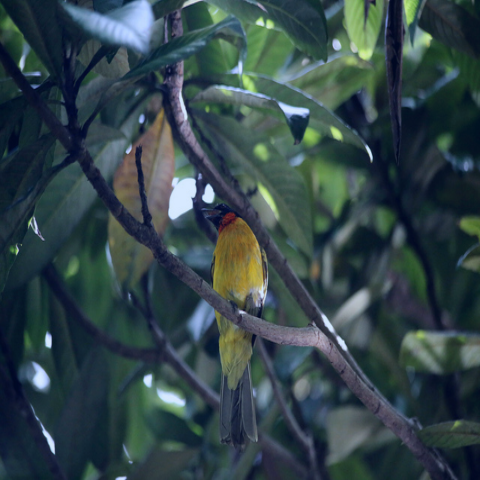}
    \end{center}
  \end{minipage}
  \begin{minipage}{0.131\linewidth}
    \begin{center}
      \includegraphics[width=\linewidth]{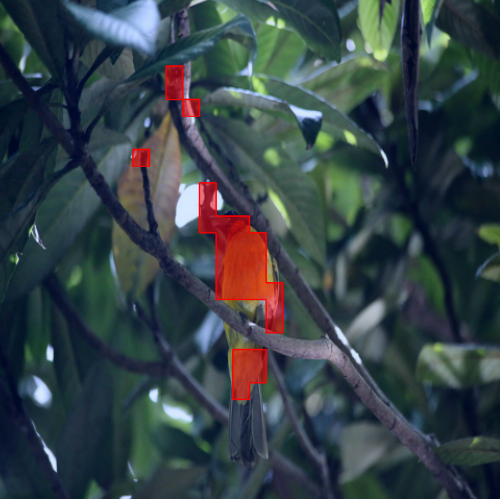}
    \end{center}
  \end{minipage}
    \begin{minipage}{0.131\linewidth}
    \begin{center}
      \includegraphics[width=\linewidth]{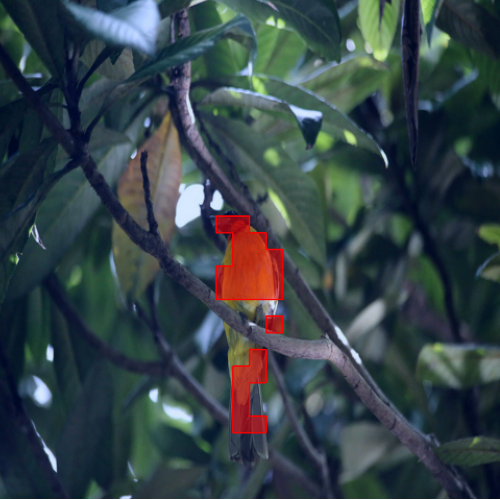}
    \end{center}
  \end{minipage}
  \begin{minipage}{0.131\linewidth}
    \begin{center}
      \includegraphics[width=\linewidth]{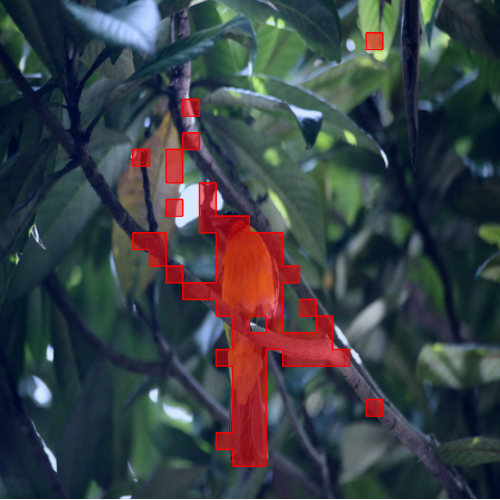}
    \end{center}
  \end{minipage}
    \begin{minipage}{0.131\linewidth}
    \begin{center}
      \includegraphics[width=\linewidth]{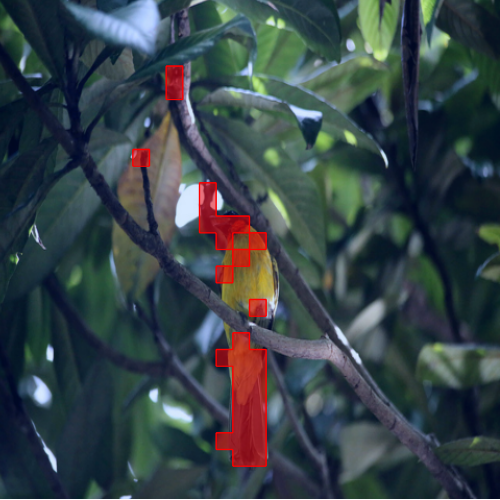}
    \end{center}
  \end{minipage}
    \begin{minipage}{0.131\linewidth}
    \begin{center}
      \includegraphics[width=\linewidth]{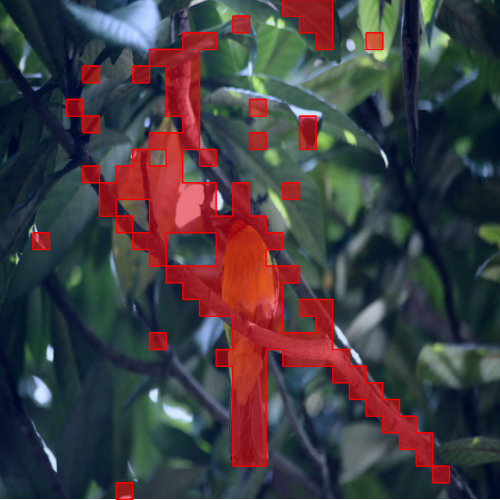}
    \end{center}
  \end{minipage}
    \begin{minipage}{0.131\linewidth}
    \begin{center}
      \includegraphics[width=\linewidth]{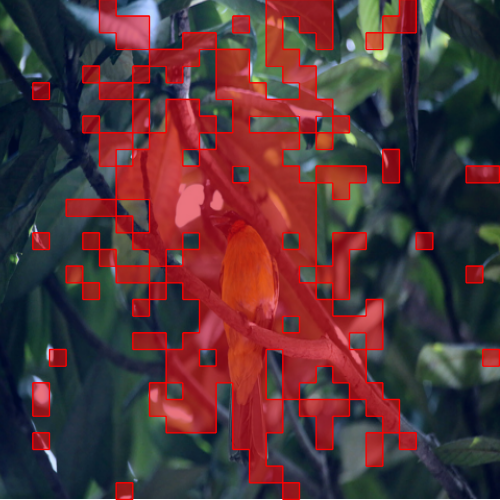}
    \end{center}
  \end{minipage}
  \begin{minipage}{0.02\linewidth}
    \begin{center}
     \begin{sideways} $k$-NN \end{sideways} 
    \end{center}
  \end{minipage}
  \begin{minipage}{0.131\linewidth}
    \begin{center}
      \includegraphics[width=\linewidth]{img.png}
      input
    \end{center}
  \end{minipage}
  \begin{minipage}{0.131\linewidth}
    \begin{center}
      \includegraphics[width=\linewidth]{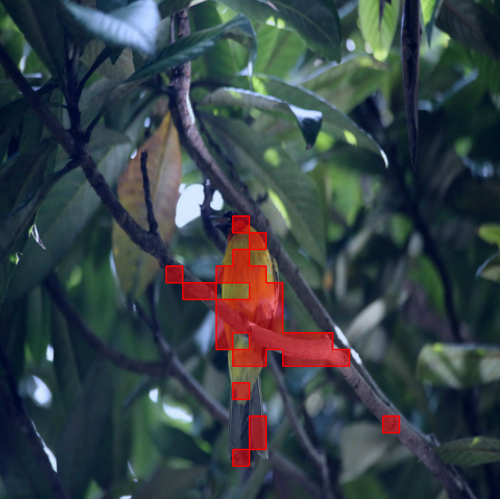}
      head0
    \end{center}
  \end{minipage}
    \begin{minipage}{0.131\linewidth}
    \begin{center}
      \includegraphics[width=\linewidth]{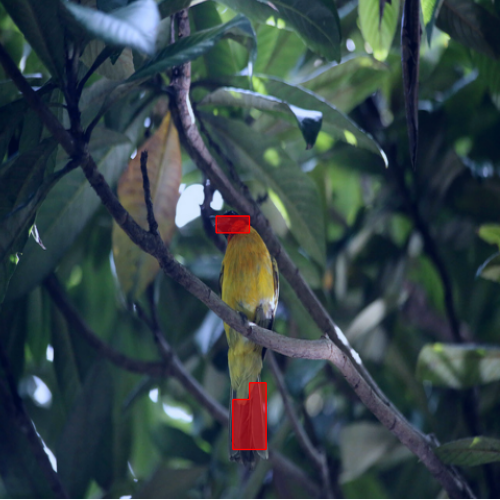}
      head1
    \end{center}
  \end{minipage}
  \begin{minipage}{0.131\linewidth}
    \begin{center}
      \includegraphics[width=\linewidth]{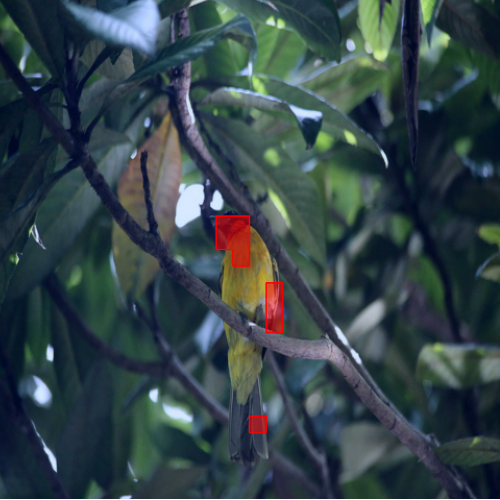}
      head2
    \end{center}
  \end{minipage}
    \begin{minipage}{0.131\linewidth}
    \begin{center}
      \includegraphics[width=\linewidth]{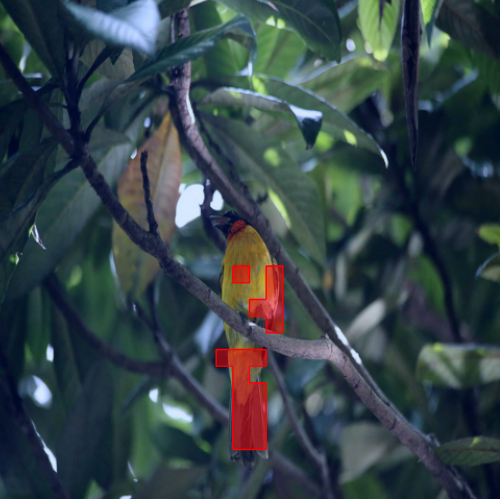}
      head3 
    \end{center}
  \end{minipage}
    \begin{minipage}{0.131\linewidth}
    \begin{center}
      \includegraphics[width=\linewidth]{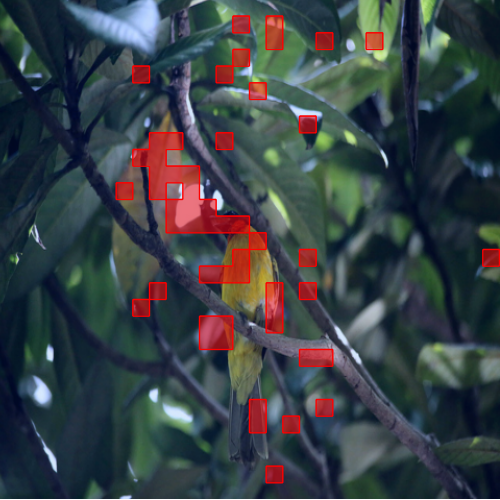}
      head4 
    \end{center}
  \end{minipage}
    \begin{minipage}{0.131\linewidth}
    \begin{center}
      \includegraphics[width=\linewidth]{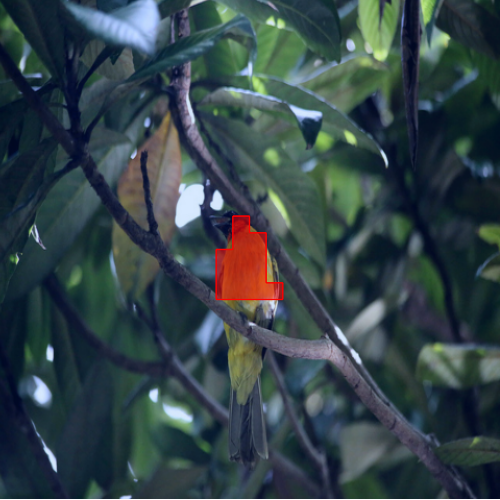}
      head5
    \end{center}
  \end{minipage}
  \caption{Self-attention heads from the last layer.}\label{Visualization1}
\end{figure}

\begin{figure}[ht!]
\begin{minipage}{0.02\linewidth}
    \begin{center}
     \begin{sideways} input \end{sideways} 
    \end{center}
  \end{minipage}
  \begin{minipage}{0.131\linewidth}
    \begin{center}
      \includegraphics[width=\linewidth]{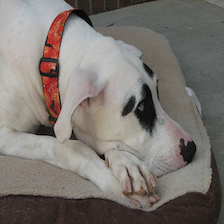}
    \end{center}
  \end{minipage}
  \begin{minipage}{0.131\linewidth}
    \begin{center}
      \includegraphics[width=\linewidth]{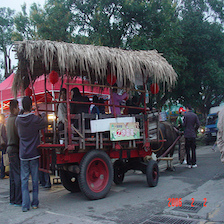}
    \end{center}
  \end{minipage}
    \begin{minipage}{0.131\linewidth}
    \begin{center}
      \includegraphics[width=\linewidth]{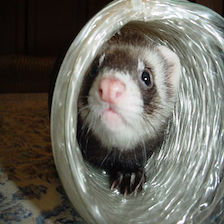}
    \end{center}
  \end{minipage}
  \begin{minipage}{0.131\linewidth}
    \begin{center}
      \includegraphics[width=\linewidth]{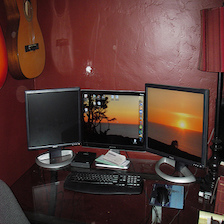}
    \end{center}
  \end{minipage}
    \begin{minipage}{0.131\linewidth}
    \begin{center}
      \includegraphics[width=\linewidth]{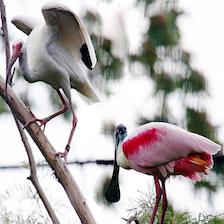}
    \end{center}
  \end{minipage}
    \begin{minipage}{0.131\linewidth}
    \begin{center}
      \includegraphics[width=\linewidth]{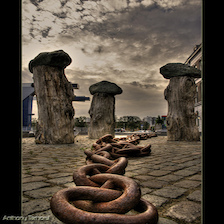}
    \end{center}
  \end{minipage}
    \begin{minipage}{0.131\linewidth}
    \begin{center}
      \includegraphics[width=\linewidth]{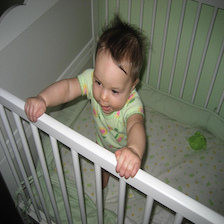}
    \end{center}
  \end{minipage}
  
\begin{minipage}{0.02\linewidth}
    \begin{center}
     \begin{sideways} dense attn \end{sideways} 
    \end{center}
  \end{minipage}
  \begin{minipage}{0.131\linewidth}
    \begin{center}
      \includegraphics[width=\linewidth]{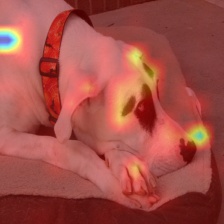}
    \end{center}
  \end{minipage}
  \begin{minipage}{0.131\linewidth}
    \begin{center}
      \includegraphics[width=\linewidth]{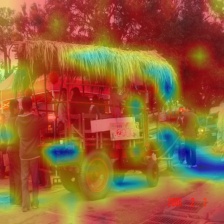}
    \end{center}
  \end{minipage}
    \begin{minipage}{0.131\linewidth}
    \begin{center}
      \includegraphics[width=\linewidth]{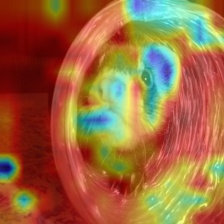}
    \end{center}
  \end{minipage}
  \begin{minipage}{0.131\linewidth}
    \begin{center}
      \includegraphics[width=\linewidth]{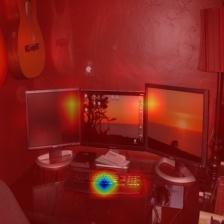}
    \end{center}
  \end{minipage}
    \begin{minipage}{0.131\linewidth}
    \begin{center}
      \includegraphics[width=\linewidth]{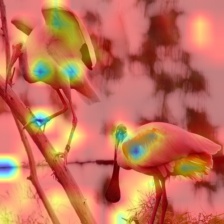}
    \end{center}
  \end{minipage}
    \begin{minipage}{0.131\linewidth}
    \begin{center}
      \includegraphics[width=\linewidth]{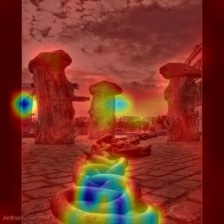}
    \end{center}
  \end{minipage}
    \begin{minipage}{0.131\linewidth}
    \begin{center}
      \includegraphics[width=\linewidth]{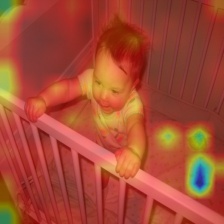}
    \end{center}
  \end{minipage}
  \begin{minipage}{0.02\linewidth}
    \begin{center}
     \begin{sideways} $k$-NN attn \end{sideways} 
    \end{center}
  \end{minipage}
  \begin{minipage}{0.131\linewidth}
    \begin{center}
      \includegraphics[width=\linewidth]{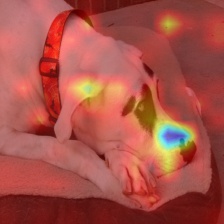}
      (a)dog
    \end{center}
  \end{minipage}
  \begin{minipage}{0.131\linewidth}
    \begin{center}
      \includegraphics[width=\linewidth]{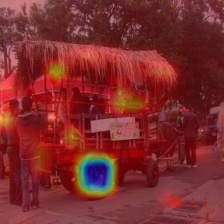}
      (b)wheel
    \end{center}
  \end{minipage}
    \begin{minipage}{0.131\linewidth}
    \begin{center}
      \includegraphics[width=\linewidth]{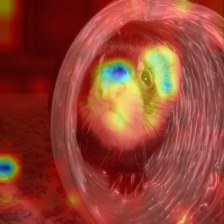}
      (c)ferret
    \end{center}
  \end{minipage}
  \begin{minipage}{0.131\linewidth}
    \begin{center}
      \includegraphics[width=\linewidth]{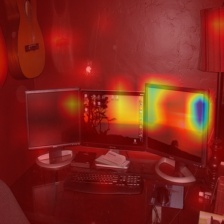}
      (d)monitor
    \end{center}
  \end{minipage}
    \begin{minipage}{0.131\linewidth}
    \begin{center}
      \includegraphics[width=\linewidth]{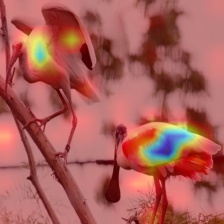}
      (e)hornbill
    \end{center}
  \end{minipage}
    \begin{minipage}{0.131\linewidth}
    \begin{center}
      \includegraphics[width=\linewidth]{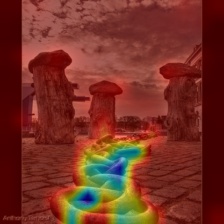}
      (f)chain 
    \end{center}
  \end{minipage}
    \begin{minipage}{0.131\linewidth}
    \begin{center}
      \includegraphics[width=\linewidth]{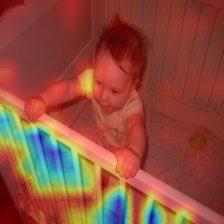}
      (g)crib
    \end{center}
  \end{minipage}
  \caption{Visualization using Transformer Attribution~\cite{chefer2020transformer}.}\label{Visualization2}
\end{figure}

\subsection{Visualization}
Figure~\ref{Visualization1} visualizes the self-attention heads from the last layer on Dino-Small~\cite{caron2021emerging}. We can see that different heads attend to different semantic regions of an image. Compared with dense attention, the $k$-NN attention filters out most irrelevant information from background regions which are similar to the foreground, and successfully concentrates on the most informative foreground regions. Images from different classes are visualized in Figure~\ref{Visualization2} using Transformer Attribution method~\cite{chefer2020transformer} on DeiT-Tiny. It can be seen that the $k$-NN attention is more concentrated and accurate, especially in the situations of cluttered background and occlusion.

\begin{table}[ht!]
\caption{Object detection and Segmentation results for Swin-Tiny and Twins-SVT-Base with/without $k$-NN attention on the COCO and ADE20K validation sets. All the models are pretrained on ImageNet-1k. }\label{ade20k}
\centering
\begin{tabular}{c|cc|cc} 
\hline
 &  \multicolumn{2}{c|}{COCO} & \multicolumn{2}{c}{ADE20K} \\
 Backbone & Method & mAP(box) & Method & mIoU \\
 \hline
 Swin-T & Mask R-CNN 3x & 46.0 & UPerNet & 44.5 \\
 Swin-T-$k$-NN  & Mask R-CNN 3x & 46.2 & UPerNet & 44.7 \\
 Twins-SVT-Base  & Mask R-CNN 1x & 45.2  & UPerNet & 47.4 \\
 Twins-SVT-Base-$k$-NN  & Mask R-CNN 1x & 45.6  & UPerNet & 47.9 \\
  \hline
\end{tabular}
\end{table}

\subsection{Object Detection and Semantic Segmentation}
To verify the effects of $k$-NN attention on object detection and semantic segmentation tasks, the widely-used COCO~\cite{lin2014microsoft} and ADE20K~\cite{zhou2019semantic} are adopted for evaluation. We adopt Swin-Tiny~\cite{liu2021swin} and Twins-SVT-Base~\cite{chu2021Twins} for comparisons due to the well released codes, and the results are shown in Table~\ref{ade20k}. From the Table we can see that by replacing the vanilla attention with our $k$-NN attention, the performance increases with almost no overhead.

\section{Conclusion}
In this paper, we propose an effective $k$-NN attention for boosting vision transformers. By selecting the most similar keys for each query to calculate the attention, it screens out the most ineffective tokens. The removal of irrelevant tokens speeds up the training. We theoretically prove its properties in speeding up  training, distilling noises without losing information, and increasing the performance by choosing a proper $k$. Several vision transformers are adopted to verify the effectiveness of the $k$-NN attention.

\clearpage
%
%
\bibliographystyle{splncs04}
\bibliography{eccv2022}

\title{Supplemental Material: \\KVT: $k$-NN Attention for Boosting Vision Transformers} 

\titlerunning{ECCV-22 submission ID \ECCVSubNumber} 
\authorrunning{ECCV-22 submission ID \ECCVSubNumber} 
\author{Anonymous ECCV submission}
\institute{Paper ID \ECCVSubNumber}

\titlerunning{KVT: $k$-NN Attention for Boosting Vision Transformers}
%
\authorrunning{Pichao Wang et al.}
\author{Pichao Wang\thanks{The first two authors contribute equally.} \and
Xue Wang\samethanks \and
Fan Wang \and
Ming Lin \and
Shuning Chang \and
Hao Li \and
Rong Jin
}
%
%
\institute{Alibaba Group\\ \email{\{pichao.wang, xue.w, fan.w, ming.l, shuning.csn, lihao.lh, jinrong.jr\}@alibaba-inc.com}\\}

\maketitle

\section{Differences with the arXiv paper: Explicit Sparse Transformer: Concentrated Attention Through Explicit Selection (EST)}
\textbf{Similarities:} Our method part is similar to EST in terms of the calculation of top-$k$.\\ \textbf{Differences:} \begin{itemize}
\item Our paper is focused not only on the methodology part, but also the deep understanding. There are many variants of Transformers in the NLP and vision community now, but few of them provide a deep and thorough analysis of their proposed methods. The proposed $k$-NN attention indeed happens to be similar to EST, which was arxived 2 years ago and we were not aware of it when conducting our research. In addition to applying the  idea to transformers and conducting extensive experiments as EST did, we provide theoretical justifications about the idea, which we think is equally or more important than the method itself, and helps with a more fundamental understanding.
\item The conclusion about how to select $k$ is different. In EST, it is found that a small $k$ is better (8 or 16), but we find a larger $k$, namely, $\geq \frac{1}{2} N$ is better ($N$ is the sequence length).
\item The motivations of these two papers are different: EST targets to get sparse attention maps while ours aims to distill noisy patches.
\item Our paper focuses on vision transformers but EST focuses on NLP tasks, even though EST applied it to the image captioning task. Since late 2020, vision transformer backbones have become very popular, and $k$-NN attention deserves a deeper analysis. Therefore, we apply the $k$-NN attention on 11 different vision transformer backbones for empirical evaluations and find it simple and effective for vision transformer backbones.
\item More analysis about the properties of $k$-NN attention in the context of vision transformer backbones are provided in our paper. Besides the $k$ selection and convergence speed as EST presented, we also define several metrics  to facilitate the analysis, e.g. layer-wise cosine similarity between tokens, layer-wise standard deviation of attention weights, ratio between the norms of residual activation and main branch, and nonlocality. We also compare it with temperature in $softmax$ and provide the visualizations.
\end{itemize}       
In summary, our paper provides deeper understanding with comprehensive analysis of the $k$-NN attention for vision transformers, which provides well-grounded knowledge advancement.

\section{Source codes of fast version $k$-NN attention in Pytorch}
The source codes of fast version $k$-NN attention in Pytorch are shown in Algorithm~\ref{code}, and we can see that the core codes of fast version $k$-NN attention is consisted of only four lines, and it can be easily imported to any architecture using fully-connected attention.
\definecolor{codegreen}{rgb}{0,0.6,0}
\definecolor{codegray}{rgb}{0.5,0.5,0.5}
\definecolor{codepurple}{rgb}{0.58,0,0.82}
\definecolor{backcolour}{rgb}{255,255,255}

\lstdefinestyle{mystyle}{
  backgroundcolor=\color{backcolour},   commentstyle=\color{codegreen},
  keywordstyle=\color{magenta},
  numberstyle=\tiny\color{codegray},
  stringstyle=\color{codepurple},
  basicstyle=\ttfamily\scriptsize,
  breakatwhitespace=false,         
  breaklines=true,                 
  captionpos=b,                    
  keepspaces=true,                 
  numbers=left,                    
  numbersep=5pt,                  
  showspaces=false,                
  showstringspaces=false,
  showtabs=false,                  
  tabsize=2
}

\lstset{style=mystyle}

\begin{algorithm}
\caption{Codes of fast version $k$-NN attention in Pytorch.}\label{code}
\begin{lstlisting}[language=Python]
class kNN-Attention(nn.Module):
    def __init__(self,dim,num_heads=8,qkv_bias=False,qk_scale=None,attn_drop=0.,proj_drop=0.,topk=100):
        super().__init__()
        self.num_heads=num_heads
        head_dim=dim//num_heads
        self.scale=qk_scale or head_dim**-0.5
        self.topk=topk

        self.qkv=nn.Linear(dim,dim*3,bias=qkv_bias)
        self.attn_drop=nn.Dropout(attn_drop)
        self.proj=nn.Linear(dim,dim)
        self.proj_drop=nn.Dropout(proj_drop)

    def forward(self,x):
        B,N,C=x.shape
        qkv=self.qkv(x).reshape(B,N,3,self.num_heads,C//self.num_heads).permute(2,0,3,1,4)
        q,k,v=qkv[0],qkv[1],qkv[2] #B,H,N,C
        attn=(q@k.transpose(-2,-1))*self.scale #B,H,N,N
        # the core code block
        mask=torch.zeros(B,self.num_heads,N,N,device=x.device,requires_grad=False)
        index=torch.topk(attn,k=self.topk,dim=-1,largest=True)[1]
        mask.scatter_(-1,index,1.)
        attn=torch.where(mask>0,attn,torch.full_like(attn,float('-inf')))
        # end of the core code block
        attn=torch.softmax(attn,dim=-1)
        attn=self.attn_drop(attn)
        x=(attn@v).transpose(1,2).reshape(B,N,C)
        x=self.proj(x)
        x=self.proj_drop(x)

        return x
\end{lstlisting}
\end{algorithm}

\section{Comparisons between slow version and fast version }\label{slowfast}
We develop two versions of $k$-NN attention, one slow version and one fast version. The $k$-NN attention is exactly defined by slow version, but its speed is extremely slow, as for each query it needs to select different $k$ keys and values, and this procedure is very slow. To speedup, we developed the CUDA version, but the speed is still slower than fast version. The fast version takes advantages of matrix multiplication and greatly speedup the computing. The speed comparisons on DeiT-Tiny using 8 V100 are illustrated in Table~\ref{speed}.

\begin{table}[H]
\centering
\begin{tabular}{l|l} 
\toprule
method & time per iteration (second)  \\ 
\midrule
 slow version (pytorch)      &     8192                \\ 
\midrule
 slow version (CUDA)       &       1.55              \\ 
\midrule
 fast version (pytorch)      &     0.45                \\
\bottomrule
\end{tabular}
\caption{The speed comparisons on DeiT-tiny for slow and fast version}\label{speed}
\end{table}

 
\section{Evaluations on CIFAR10 or CIFAR100.}
As vision transformers are data-hungry, directly training vision transformer backbones from scratch on small-size datasets such as CIFAR10 or CIFAR100 would yield much worse performances compared with ConvNets. Following the paradigm and codes of the NIPS2021 paper ``Efficient Training of Visual Transformers with Small Datasets", we briefly conducted experiments on CIFAR10 and CIFAR100 using Swin-T and T2T-ViT-14 with $k$-NN attention as shown in Table~\ref{cifar}. Adding $k$-NN attention brings much larger performance gain in the scratch training (ST) due to its faster convergence speed, while the gain in the setting of ImageNet-1k pretraining and CIFAR finetuning (FT) is not as large.
\begin{table}[H]
     \centering
        \begin{tabular}{lccccc}
        \toprule
            Model & C10 (ST) & C100 (ST) & C10 (FT) & C100 (FT)\\
            \midrule
            Swin-T      & 83.9 & 66.2  & 98.4 & 88.4\\
            Swin-T$\rightarrow$ k-NN Attn & 84.5 & 67.1 & 98.6 & 88.7\\
            T2T-VIT-14     & 87.6 & 68.0 & 98.5 & 87.7\\
            T2T-VIT-14$\rightarrow$ k-NN Attn & 88.2 & 68.8 & 98.8 & 88.1\\
        \bottomrule
        \end{tabular}
        \caption{Results on CIFAR10 and CIFAR100 (100 epochs).} \label{cifar}
    \end{table}

\section{Proof}\label{sec:proof}

{\bf Notations.} Throughout this appendix, we denote $x_i$ as $i$-th element of vector $\bm{x}$, $\bm{W}_{ij}$ as the element at $i$-th row and $j$-th column of matrix $\bm{W}$, and $\bm{W}_j$ as the $j$-th row of matrix $\bm{W}$. Moreover, we denote $\bm{x}_i$ as the $i$-th patch (token) of the inputs with $\bm{x}_i = \bm{X}_i$.\\

{\bf Proof for Lemma \ref{lem:1}}
We first give the formal statement of Lemma~\ref{lem:1}.

\begin{lemma}[Formal statement of Lemma~\ref{lem:1}]
Let $\hat{\bm{V}}^{knn}_l$ be the $l$-th row of the $\hat{\bm{V}}^{knn}$ and $\textrm{Var}_{\bm{a}_l}(\bm{x}) = \mathbbm{E}_{\bm{a}_l}[\bm{x}^{\top}\bm{x}]-\mathbbm{E}_{\bm{a}_l}[\bm{x}^{\top}]\mathbbm{E}_{\bm{a}_l}[\bm{x}]$ with $\mathbbm{E}_{\bm{a}_l}[\bm{x}] = \sum_{t=1}^n a_{lt}\bm{x}_t$.
Then for any $i,j = 1,2,...,n$, we have
\begin{align}
	\frac{\partial \hat{\bm{V}}_l}{\partial W_{\bm{Q},ij}} =
 x_{li}\bm{W}_{\bm{K},j}^{\top}\textrm{Var}_{a_l}(\bm{x})\bm{W}_{\bm{V}} \propto  \textrm{Var}_{a_l}(\bm{x})\notag
\end{align}
and
\begin{align}
	\frac{\partial \hat{\bm{V}}_l}{\partial W_{\bm{K},ij}} =
 x_{li}\bm{W}_{\bm{Q},j}^{\top}\textrm{Var}_{a_l}(\bm{x})\bm{W}_{\bm{V}} \propto  \textrm{Var}_{a_l}(\bm{x}).\notag
\end{align}
The same is true for $\hat{\bm{V}}$ of the fully-connected self-attention.
%
\end{lemma}

\begin{proof}
Let's first consider the derivative of $\hat{\bm{V}}_l$ over $W_{\bm{Q},ij}$. Via some  algebraic computation, we have  
	\begin{align}
	\frac{\partial \hat{\bm{V}}_l}{\partial W_{\bm{Q},ij}} =   \frac{\partial (\bm{a}_l\bm{V})}{\partial W_{\bm{Q},ij}}
=	\displaystyle\sum_{t=1}^n\displaystyle a_{lt}\left(\frac{\partial \mathcal{T}_{l}^{knn}(t)}{\partial W_{\bm{Q},ij}}-\displaystyle\sum_{k_1 = 1}^n a_{lk_1}\frac{\partial \mathcal{T}_{l}^{knn}(k_1)}{\partial W_{\bm{Q},ij}}\right)\bm{x}_t\bm{W}_{\bm{V}}
\label{eq:lem:proof:1},
	\end{align}
where we denote $\mathcal{T}_{l}^{knn}(k)$ as follow for shorthand:
\begin{align}
\mathcal{T}_l^{knn}(k_1) = \begin{cases}
 	\bm{x}_l\bm{W}_{\bm{Q}}\bm{W}_{\bm{K}}^{\top}\bm{x}_{k_1}^{\top},\  &\textrm{if patch $k_1$ is selected in row $l$}\\
 -\infty, \  &\textrm{otherwise}
 \end{cases}\notag
\end{align}
Let denote set $\mathcal{S}\doteq \{i: \textrm{patch $i$ is selected in row $l$}\}$ and then we consider the right-hand-side of \eqref{eq:lem:proof:1}.
\begin{align}
\eqref{eq:lem:proof:1} =& \sum_{t\in \mathcal{S}}^na_{lt}\left(\frac{\partial\left(\bm{x}_i\bm{W}_{\bm{Q}}\bm{W}_{\bm{K}}^{\top}\bm{x}_k^{\top}\right)}{\partial W_{\bm{Q},ij}}-\sum_{k_1\in \mathcal{S}}a_{lk_1}\frac{\partial \left(\bm{x}_i\bm{W}_{\bm{Q}}\bm{W}_{\bm{K}}^{\top}\bm{x}_{k_1}^{\top}\right)}{\partial W_{\bm{Q},ij}}\right)\bm{x}_t\bm{W}_{\bm{V}}	\notag\\
=& \sum_{t\in \mathcal{S}}^na_{lt}\left(x_{1i}\bm{x}_{t}\bm{W}_{\bm{K},j}-\sum_{k_1\in \mathcal{S}}a_{lk_1}x_{li}\bm{x}_{k_1}\bm{W}_{\bm{K},j}\right)\bm{x}_t\bm{W}_{\bm{V}}\notag\\
=& \underbrace{\sum_{t\in \mathcal{S}}^na_{lt}x_{li}\bm{x}_t\bm{W}_{\bm{K},j}\bm{x}_t\bm{W}_{\bm{V}}}_{(a)}-	\underbrace{\sum_{t\in \mathcal{S}}^na_{lt}\bm{x}_t\bm{W}_{\bm{V}}}_{(b)}\cdot\underbrace{\sum_{k_1\in \mathcal{S}}a_{lk_1}x_{li}\bm{x}_{k_1}\bm{W}_{\bm{K},j}}_{(c)}\label{eq:lem:proof:2}.
\end{align}

Since $\bm{a}_l$ is the $l$-th row of the attention matrix, we have $a_{lt}\ge0$ and $\sum_{t}a_{lt} = 1$. It is possible to treat terms $(a)$, $(b)$ and $(c)$ as the expectation of some quantities over $t$ replicates with probability $a_{lt}$.  Then \eqref{eq:lem:proof:2} can be further simplified as 

\begin{align}
\eqref{eq:lem:proof:2}&=\mathbbm{E}_{\bm{a}_{l}}[x_{li}\bm{x}\bm{W}_{\bm{K},j}\cdot\bm{x}\bm{W}_{\bm{V}}] - \mathbbm{E}_{\bm{a}_l}[\bm{x}\bm{W}_{\bm{K},j}]\cdot \mathbbm{E}_{\bm{a}_l}[x_{li}\bm{x}\bm{W}_{\bm{V}}]\notag\\
&=x_{li}\left(\mathbbm{E}_{\bm{a}_{l}}[\bm{W}_{\bm{K},j}^{\top}\bm{x}^{\top}\cdot\bm{x}\bm{W}_{\bm{V}}] - \mathbbm{E}_{\bm{a}_l}[\bm{W}_{\bm{K},j}^{\top}\bm{x}^{\top}]\cdot\mathbbm{E}_{\bm{a}_l}[\bm{x}\bm{W}_{\bm{V}}]\right)\notag\\
&=x_{li}\bm{W}_{\bm{K},j}^{\top}\left(\mathbbm{E}_{\bm{a}_{l}}[\bm{x}^{\top}\bm{x}] - \mathbbm{E}_{\bm{a}_l}[\bm{x}^{\top}]\cdot \mathbbm{E}_{\bm{a}_l}[\bm{x}_t]\right)\bm{W}_{\bm{V}}\notag\\
&=x_{li}\bm{W}_{\bm{K},j}^{\top}\textrm{Var}_{\bm{a}_l}(\bm{x})\bm{W}_{\bm{V}}\label{eq:lem:proof:3},
\end{align}
where the second equality uses the fact that $\bm{x}_t\bm{W}_{\bm{K},j}$ is a scalar.

Combing \eqref{eq:lem:proof:1}-\eqref{eq:lem:proof:3}, we have
\begin{align}
	\frac{\partial \hat{\bm{V}}_l}{\partial W_{\bm{Q},ij}} =
 x_{li}\bm{W}_{\bm{K},j}^{\top}\textrm{Var}_{\bm{a}_l}(\bm{x})\bm{W}_{\bm{V}} \propto  \textrm{Var}_{\bm{a}_l}(\bm{x}).\label{eq:final_final_11}
\end{align}
Due the symmetric on $\bm{Q}$ and $\bm{K}$, we can follow the similar procedure to show 
\begin{align}
	\frac{\partial \hat{\bm{V}}_l}{\partial W_{\bm{K},ij}} =
 x_{li}\bm{W}_{\bm{Q},j}^{\top}\textrm{Var}_{\bm{a}_l}(\bm{x})\bm{W}_{\bm{V}} \propto  \textrm{Var}_{\bm{a}_l}(\bm{x}).\label{eq:final_final_12}
\end{align}
Finally, by setting $k = n$,  one may verify that equations \eqref{eq:final_final_11} and \eqref{eq:final_final_12} also hold for fully-connected self-attention.

\end{proof}

%
%

{\bf Proof for Lemma \ref{lem:3}}
Before given the formal statement of the Lemma \ref{lem:3}, we first show the assumptions.

{\bf Assumption 2}
\begin{enumerate}
    \item The token $\bm{x}_i$ is the sub-gaussian random vector with mean $\bm{\mu}_i$ and variance $(\sigma^2/d) I$ for $i=1,2,...,n$.
    \item $\bm{\mu}$ follows a discrete distribution with finite values $\bm{\mu}\in \mathcal{V}$. Moreover, there exist $0<\nu_{1},0<\nu_{2}<\nu_4$ such that a) $\|\bm{\mu}_i\| = \nu_1$, and b) $\bm{\mu}_i\bm{W}_{\bm{Q}}\bm{W}_{\bm{K}}^T\bm{\mu}_i\in [\nu_2,\nu_4]$  for all $i$ and $| \bm{\mu}_i\bm{W}_{\bm{Q}}\bm{W}_{\bm{K}}^{\top}\bm{\mu}_j^{\top}|\le \nu_2$ for all $\bm{\mu}_{i}\ne \bm{\mu}_{j}\in\mathcal{V}$. \item $\bm{W}_V$ and $\bm{W}_{\bm{Q}}\bm{W}_{\bm{K}}^{\top}$ are element-wise bounded with $\nu_5$ and $\nu_6$ respectively, that is, $|\bm{W}_V^{(ij)}|\le \nu_5$ and $|(\bm{W}_{\bm{Q}}\bm{W}_{\bm{K}}^{\top})^{(ij)}|\le \nu_6$, for all $i,j$ from 1 to $d$.
\end{enumerate}
In Assumption 2 we ensure that for a given query patch, the difference between the clustering center and noises are large enough to be distinguished. 




\begin{lemma}[formal statement of Lemma \ref{lem:3}]
Let patch $\bm{x}_i$ be $\sigma^2$-subgaussian random variable with mean $\bm{\mu}_i$ and there are $k_1$ patches out of all $k$ patches follows the same clustering center of query   $l$. Per Assumption 2, when $\sqrt{d}\ge 3(\psi(\delta,d)+\nu_2+\nu_4)$, then with probability $1-5\delta$, we have

\begin{align}
&\left\|\frac{\sum_{i=1}^k\exp\left(\frac{1}{\sqrt{d}}\bm{x}_l\bm{W}_{\bm{Q}}\bm{W}^{\top}_k\bm{x}_i\right)\bm{x}_i\bm{W}_V}{\sum_{j=1}^k\exp\left(\frac{1}{\sqrt{d}}\bm{x}_l\bm{W}_{\bm{Q}}\bm{W}_{\bm{K}}^{\top}\bm{x}_j\right)}-\bm{\mu}_l\bm{W}_V\right\|_{\infty} \\& \le 4\exp\left(\frac{\psi(\delta,d)}{\sqrt{d}}\right)\sigma\nu_5 \sqrt{\frac{2}{dk}\log\left(\frac{2d}{\delta}\right)}\notag\\
&+  \left[8\exp\left(\frac{\nu_2-\nu_4+\psi(\delta,d)}{\sqrt{d}}\right)-\left(7+\exp\left(\frac{\nu_2-\nu_4+\psi(\delta,d)}{\sqrt{d}}\right)\right)\frac{k_1}{k}\right]\|\bm{\mu}_1\bm{W}_V\|_{\infty}\notag,
\end{align}
where $\psi(\delta,d) = 2\sigma\nu_1\nu_6\sqrt{2\log\left(\frac{1}{\delta}\right)}+ 2\sigma^2\nu_6\log\left(\frac{d}{\delta}\right)$.\\
\end{lemma}
\begin{proof}
Without loss of generality, we assume the first $k$ patch are the top-$k$ selected patches.
From Assumption {\bf 2.1}, we can decompose $\bm{x}_i = \bm{\mu}_i+\bm{h}_i$, $i=1,2,...,k$, where $\bm{h}_i$ is the sub-gaussian random vector with zero mean. We then analyze the numerator part.
\begin{align}
	&\sum_{i=1}^k\exp\left(\frac{1}{\sqrt{d}}\bm{x}_l\bm{W}_{\bm{Q}}\bm{W}^{\top}_k\bm{x}_i\right)\bm{x}_i\bm{W}_V\notag\\
	=&\overbrace{\sum_{i=1}^k\exp\left(\frac{1}{\sqrt{d}}\bm{\mu}_l\bm{W}_{\bm{Q}}\bm{W}_{\bm{K}}^{\top}\bm{\mu}_i^{\top}\right)\bm{\mu}_{i}\bm{W}_v}^{(a)}
	+\overbrace{\sum_{i=1}^k\exp\left(\frac{1}{\sqrt{d}}\bm{x}_l\bm{W}_{\bm{Q}}\bm{W}_{\bm{K}}^{\top}\bm{x}_i^{\top}\right)\bm{h}_{i}\bm{W}_v}^{(b)}\notag\\
	+&\overbrace{\sum_{i=1}^k\left[\exp\left(\frac{1}{\sqrt{d}}\bm{x}_l\bm{W}_{\bm{Q}}\bm{W}^{\top}_k\bm{x}_i\right)-\exp\left(\frac{1}{\sqrt{d}}\bm{\mu}_l\bm{W}_{\bm{Q}}\bm{W}_{\bm{K}}^{\top}\bm{\mu}_i^{\top}\right)\right]\bm{\mu}_i\bm{W}_v}^{(c)}\label{eq:15}.
\end{align}
Below we will bound $(a)$, $(b)$ and $(c)$ separately.\\

{\bf\noindent Upper bound for $(a)$.}
Let denote index set $\mathcal{S}_1 = \{i: \bm{\mu}_1 = \bm{\mu}_i, \ i=1,2,...,k\}$. We then have
\begin{align}
&\left\|(a) - \sum_{i\in\mathcal{S}_1}\exp\left(\frac{1}{\sqrt{d}}\bm{x}_1\bm{W}_{\bm{Q}}\bm{W}_{\bm{K}}^{\top}\bm{x}_i^{\top}\right)\bm{\mu}_1\bm{W}_V\right\|_{\infty}\notag\\
\le & (k-|\mathcal{S}_1|)\max_{i}\left\{\exp\left(\frac{1}{\sqrt{d}}\bm{x}_1\bm{W}_{\bm{Q}}\bm{W}_{\bm{K}}^{\top}\bm{x}_i^{\top}\right)\right\}\|\bm{\mu}_1\bm{W}_V\|_{\infty}\notag\\
\le & (k-k_1)\exp\left(\frac{\nu_2}{\sqrt{d}}\right)\|\bm{\mu}_1\bm{W}_V\|_{\infty},\label{eq:16}
\end{align}
where last inequality is from the Assumption {\bf 2.2}. \\

{\bf\noindent Upper bound for $(b)$.} Since each dimension in $\bm{h}_l$ is the i.i.d random vector with zero mean variance $\sigma^2/d$ based on Assumption {\bf 2.1}, we can use Hoeffding Inequality to derive the following result holds with probability $1-\delta$.

\begin{align}
\left\|\sum_{i=1}^k\exp\left(\frac{1}{\sqrt{d}}\bm{x}_1\bm{W}_{\bm{Q}}\bm{W}_{\bm{K}}^{\top}\bm{x}_i^{\top}\right)\bm{h}_i\bm{W}_V\right\|_{\infty}\le \sigma\nu_5\sqrt{\frac{2(k_1U_1^2+(k-k_1)U_2^2)}{d}\log\left(\frac{2d}{\delta}\right)},
\end{align}
where $$U_1 = \max_{i\in\mathcal{S}_1}\left\{\exp\left(\frac{1}{\sqrt{d}}\bm{x}_1\bm{W}_{\bm{Q}}\bm{W}_{\bm{K}}^{\top}\bm{x}_i^{\top}\right) \right\}$$ and  $$U_2 = \max_{i\notin\mathcal{S}_1}\left\{\exp\left(\frac{1}{\sqrt{d}}\bm{x}_1\bm{W}_{\bm{Q}}\bm{W}_{\bm{K}}^{\top}\bm{x}_i^{\top}\right) \right\}.$$ 

We then build the upper bound for $U_1$ and $U_2$. 
Since $\bm{x}_i = \bm{\mu}_i+\bm{h}_i$ for $i =1,2,...,k$, we have
\begin{align}
	&|\bm{x}_1\bm{W}_{\bm{Q}}\bm{W}_{\bm{K}}^{\top}\bm{x}_i^{\top}|\le |\bm{\mu}_1\bm{W}_{\bm{Q}}\bm{W}_{\bm{K}}^{\top}\bm{\mu}_i^{\top}| 
	+ |\bm{\mu}_1\bm{W}_{\bm{Q}}\bm{W}_{\bm{K}}^{\top}\bm{h}_i^{\top}| \\
	&+|\bm{h}_1\bm{W}_{\bm{Q}}\bm{W}_{\bm{K}}^{\top}\bm{\mu}_i^{\top}|
	+|\bm{h}_1\bm{W}_{\bm{Q}}\bm{W}_{\bm{K}}^{\top}\bm{h}_i^{\top}|\notag
\end{align}
Via Assumption {\bf 2.3} and Hoeffding Inequality, with probability $1-4\delta$, the follow results hold.
\begin{align}
|\bm{\mu}_1\bm{W}_{\bm{Q}}\bm{W}_{\bm{K}}^{\top}\bm{h}_i^{\top}|
&\le\sigma\nu_1\nu_6\sqrt{2\log\left(\frac{1}{\delta}\right)}\label{eq:19}\\
|\bm{h}_1\bm{W}_{\bm{Q}}\bm{W}_{\bm{K}}^{\top}\bm{\mu}_i^{\top}|
&\le \sigma\nu_1\nu_6\sqrt{2\log\left(\frac{1}{\delta}\right)}\label{eq:20}\\\
|\bm{h}_1\bm{W}_{\bm{Q}}\bm{W}_{\bm{K}}^{\top}\bm{h}_i^{\top}|
&\le 2\sigma^2\nu_6\log\left(\frac{d}{\delta}\right)\label{eq:21}
\end{align}
and we denote $\psi(\delta,d) = 2\sigma\nu_1\nu_6\sqrt{2\log\left(\frac{1}{\delta}\right)}+ 2\sigma^2\nu_6\log\left(\frac{d}{\delta}\right)$ for shorthand and then we have
\begin{align}
U_1&\le \exp\left[\frac{1}{\sqrt{d}}\left(\nu_2+\psi(\delta,d)\right)\right]\notag\\
U_2&\le \exp\left[\frac{1}{\sqrt{d}}\left(\nu_4+\psi(\delta,d)\right)\right]\notag
\end{align}

As a result, with a probability $1-5\delta$, we have: 
\begin{align}
\|(b)\|_{\infty}\le \sigma\nu_5\exp\left(\frac{\psi(\delta,d)}{\sqrt{d}}\right)\sqrt{\frac{2(k_1\exp\left(\frac{2\nu_2}{\sqrt{d}}\right)+(k-k_1)\exp\left(\frac{2\nu_4}{\sqrt{d}}\right))}{d}\log\left(\frac{2d}{\delta}\right)}.\label{eq:23}
\end{align}

{\bf Upper bound for $(c)$}.
\begin{align}
\|(c)\|_{\infty}&\le \left|\sum_{i=1}^k\left[\exp\left(\frac{1}{\sqrt{d}}\bm{x}_1\bm{W}_{\bm{Q}}\bm{W}_{\bm{K}}^{\top}\bm{x}_i^{\top}\right)-\exp\left(\frac{1}{\sqrt{d}}\bm{\mu}_1\bm{W}_{\bm{Q}}\bm{W}_{\bm{K}}^{\top}\bm{\mu}_i^{\top}\right)\right]\right|\|\bm{\mu}_1\bm{W}_V\|_{\infty}\notag
\end{align}
and it implies
\begin{align}
&\frac{\|(c)\|_{\infty}}{\|\bm{\mu}_1\bm{W}_V\|_{\infty}}\notag\\
&\le \left|\sum_{i\notin\mathcal{S}_1}\left[\left(\exp\left(\frac{1}{\sqrt{d}}\left(\bm{x}_1\bm{W}_{\bm{Q}}\bm{W}_{\bm{K}}^{\top}\bm{x}_i^{\top}-\bm{\mu}_1\bm{W}_{\bm{Q}}\bm{W}_{\bm{K}}^{\top}\bm{\mu}_i^{\top}\right)\right)-1\right)\right]\right|\exp\left(\frac{\nu_2}{\lambda}\right)\notag\\
&+ \left|\sum_{i\in\mathcal{S}_1}\left[\left(\exp\left(\frac{1}{\sqrt{d}}\left(\bm{x}_1\bm{W}_{\bm{Q}}\bm{W}_{\bm{K}}^{\top}\bm{x}_i^{\top}-\bm{\mu}_1\bm{W}_{\bm{Q}}\bm{W}_{\bm{K}}^{\top}\bm{\mu}_i^{\top}\right)\right)-1\right)\right]\right|\exp\left(\frac{\nu_5}{\lambda}\right)\label{eq:24}.
\end{align}
Combine \eqref{eq:24} with \eqref{eq:19}-\eqref{eq:21} and we have with probability $1-4\delta$:
\begin{align}
	\|(c)\|_{\infty}&\le \left|\exp\left(\frac{\psi(\delta,d)}{\sqrt{d}}\right)-1\right|\left[(k-k_1)\exp\left(\frac{\nu_2}{\sqrt{d}}\right)+k_1\exp\left(\frac{\nu_5}{\sqrt{d}}\right)\right]\|\bm{\mu}_1\bm{W}_V\|_{\infty}\label{eq:25}.
\end{align}


From \eqref{eq:15}, \eqref{eq:16}, \eqref{eq:23} and \eqref{eq:25}, with probability $1-5\delta$, we have
\begin{align}
&\left\|\sum_{i=1}^k\exp\left(\frac{1}{\sqrt{d}}\bm{x}_1\bm{W}_{\bm{Q}}\bm{W}_{\bm{K}}^{\top}\bm{x}_i^{\top}\right)\bm{x}_{i}\bm{W}_V-\sum_{i\in\mathcal{S}_1}\exp\left(\frac{1}{\sqrt{d}}\bm{\mu}_1\bm{W}_{\bm{Q}}\bm{W}_{\bm{K}}^{\top}\bm{\mu}_i^{\top}\right)\bm{\mu}_1\bm{W}_V\right\|_{\infty}\notag\\
&\le  \exp\left(\frac{\psi(\delta,d)}{\sqrt{d}}\right)\left[\tau_1(k,k_1)\|W_v\bm{\mu}_1\|_{\infty}+\sigma\nu_5 \sqrt{\frac{2\tau_2(k,k_1)}{d}\log\left(\frac{2d}{\delta}\right)}\right],\label{eq:26}
\end{align}
where $\tau_1(k,k_1) = (k-k_1)\exp\left(\frac{\nu_2}{\sqrt{d}}\right)+k_1\exp\left(\frac{\nu_5}{\sqrt{d}}\right)$ and $\tau_2(k,k_1) = k_1\exp\left(\frac{2\nu_2}{\sqrt{d}}\right)+(k-k_1)\exp\left(\frac{2\nu_4}{\sqrt{d}}\right)$.\\

Now we consider the upper bound the denominator part.

\begin{align}
&\left|\sum_{i=1}^k\exp\left(\frac{1}{\sqrt{d}}\bm{x}_1\bm{W}_{\bm{Q}}\bm{W}_{\bm{K}}^{\top}\bm{x}_i^{\top}\right) - \sum_{i=1}^k\exp\left(\frac{1}{\sqrt{d}}\bm{\mu}_1\bm{W}_{\bm{Q}}\bm{W}_{\bm{K}}^{\top}\bm{\mu}_i^{\top}\right)\right|\notag\\
\le& \overbrace{\left|\sum_{i=1}^k\exp\left(\frac{1}{\sqrt{d}}\bm{\mu}_1\bm{W}_{\bm{Q}}\bm{W}_{\bm{K}}^{\top}\bm{\mu}_i^{\top}\right)\left[\exp\left(\frac{1}{\sqrt{d}}\left(\bm{x}_1\bm{W}_{\bm{Q}}\bm{W}_{\bm{K}}^{\top}\bm{x}_i^{\top}-\bm{\mu}_1\bm{W}_{\bm{Q}}\bm{W}_{\bm{K}}^{\top}\bm{\mu}_i^{\top}\right)\right) - 1\right]\right|}^{(g_1)}.\label{eq:27}
\end{align}
Via Assumption {\bf 2.2}, \eqref{eq:19}-\eqref{eq:21} and the definition of $\psi(\delta,d)$, with probability $1-5\delta$ the follow results hold.
\begin{align}
(g_1)&\le \exp\left(\frac{\nu_4}{\sqrt{d}}\right)\left|\sum_{i=1}^k\left[\exp\left(\frac{1}{\sqrt{d}}\left[\bm{x}_1\bm{W}_{\bm{Q}}\bm{W}_{\bm{K}}^{\top}\bm{x}_i^{\top}-\bm{\mu}_1\bm{W}_{\bm{Q}}\bm{W}_{\bm{K}}^{\top}\bm{\mu}_i^{\top}\right]\right) - 1\right]\right|\notag\\
&\le \exp\left(\frac{\nu_4}{\sqrt{d}}\right)\left|\sum_{i=1}^k\left[\exp\left(\frac{\psi(\delta,d)}{\sqrt{d}}\right) - 1\right]\right|\notag\\
&\le k\exp\left(\frac{\nu_4}{\sqrt{d}}\right)\left[\exp\left(\frac{\psi(\delta,d)}{\sqrt{d}}\right) - 1\right].\label{eq:28}
\end{align}

Combining \eqref{eq:27}, \eqref{eq:28} and Assumption {\bf 2.2}, we have
\begin{align}
& \sum_{i=1}^k\exp\left(\frac{1}{\sqrt{d}}\bm{x}_1\bm{W}_{\bm{Q}}\bm{W}_{\bm{K}}^{\top}\bm{x}_i^{\top}\right)\notag\\
\ge & \sum_{i=1}^k\exp\left(\frac{1}{\sqrt{d}}\bm{\mu}_1\bm{W}_{\bm{Q}}\bm{W}_{\bm{K}}^{\top}\bm{\mu}_i^{\top}\right)-k\exp\left(\frac{\nu_4}{\sqrt{d}}\right)\left(\exp\left(\frac{\psi(\delta,d)}{\sqrt{d}}\right) - 1\right)\notag\\
\ge &  k\exp\left(-\frac{\nu_2}{\sqrt{d}}\right)-k\exp\left(\frac{\nu_4}{\sqrt{d}}\right)\left(\exp\left(\frac{\psi(\delta,d)}{\sqrt{d}}\right) - 1\right).\label{eq:32}
\end{align}
When $\sqrt{d} \ge  3(\psi(\delta,d)+\nu_2+\nu_4)$, one may verify
\begin{align}
	&\exp\left(\frac{\psi(\delta,d)+\nu_2+\nu_4}{\sqrt{d}}\right) - \exp\left(\frac{\nu_2+\nu_4}{\sqrt{d}}\right)\le \exp\left(\frac{1}{3}\right)-1\approx 0.39<1\notag\\
	\Rightarrow&\exp\left(\frac{\psi(\delta,d)}{\sqrt{d}}\right) - 1-\exp\left(-\frac{\nu_2+\nu_4}{\sqrt{d}}\right)<0\notag\\
	\Rightarrow&\exp\left(-\frac{\nu_2}{\sqrt{d}}\right)-\exp\left(\frac{\nu_4}{\sqrt{d}}\right)\left(\exp\left(\frac{\psi(\delta,d)}{\sqrt{d}}\right) - 1\right)>0\notag\\
	\Rightarrow&k\exp\left(-\frac{\nu_2}{\sqrt{d}}\right)-k\exp\left(\frac{\nu_4}{\sqrt{d}}\right)\left(\exp\left(\frac{\psi(\delta,d)}{\sqrt{d}}\right) - 1\right)>0.\label{eq:33}
\end{align}

Finally, via  \eqref{eq:26}, we have
\begin{align}
&\left\|\frac{\sum_{i=1}^k\exp\left(\frac{1}{\sqrt{d}}\bm{x}_1\bm{W}_{\bm{Q}}\bm{W}_{\bm{K}}^{\top}\bm{x}_i^{\top}\right)}{\sum_{j=1}^k\exp\left(\frac{1}{\sqrt{d}}\bm{x}_1\bm{W}_{\bm{Q}}\bm{W}_{\bm{K}}^{\top}\bm{x}_j^{\top}\right)}\bm{x}_i\bm{W}_V-\bm{\mu}_1\bm{W}_V\right\|_{\infty}\cdot 
\sum_{j=1}^k\exp\left(\frac{1}{\sqrt{d}}\bm{x}_1\bm{W}_{\bm{Q}}\bm{W}_{\bm{K}}^{\top}\bm{x}_j^{\top}\right)\notag\\
=&	\left\|\sum_{j=1}^k\exp\left(\frac{1}{\sqrt{d}}\bm{x}_1\bm{W}_{\bm{Q}}\bm{W}_{\bm{K}}^{\top}\bm{x}_i^{\top}\right)\bm{x}_{i}\bm{W}_V-\sum_{i=1}^k\exp\left(\frac{1}{\sqrt{d}}\bm{x}_1\bm{W}_{\bm{Q}}\bm{W}_{\bm{K}}^{\top}\bm{x}_j^{\top}\right)\bm{\mu}_1\bm{W}_V\right\|_{\infty}\notag\\
\le &	\left\|\sum_{i=1}^k\exp\left(\frac{1}{\sqrt{d}}\bm{x}_1\bm{W}_{\bm{Q}}\bm{W}_{\bm{K}}^{\top}\bm{x}_i^{\top}\right)\bm{x}_{i}\bm{W}_V-\sum_{i\in\mathcal{S}_1}\exp\left(\frac{1}{\sqrt{d}}\bm{\mu}_1\bm{W}_{\bm{Q}}\bm{W}_{\bm{K}}^{\top}\bm{\mu}_i^{\top}\right)\bm{\mu}_1\bm{W}_V\right\|_{\infty}\notag\\
+ &	\left\|\sum_{i\notin\mathcal{S}_1}\exp\left(\frac{1}{\sqrt{d}}\bm{\mu}_1\bm{W}_{\bm{Q}}\bm{W}_{\bm{K}}^{\top}\bm{\mu}_i^{\top}\right)\bm{\mu}_1\bm{W}_V\right\|_{\infty}\notag\\
\le & \exp\left(\frac{\psi(\delta,d)}{\lambda}\right)\left[2\tau_1(k,k_1)\|\bm{\mu}_1\bm{W}_V\|_{\infty}+\sigma\nu_5 \sqrt{\frac{2\tau_2(k,k_1)}{d}\log\left(\frac{2d}{\delta}\right)}\right]\notag\\
-&k_1\exp\left(\frac{\nu_4}{\sqrt{d}}\right)\|\bm{\mu}_1\bm{W}_V\|_{\infty}
+(k-k_1)\exp\left(\frac{\nu_2}{\sqrt{d}}\right)\|\bm{\mu}_1\bm{W}_V\|_{\infty}.\label{eq:31_sub}
\end{align}

Per \eqref{eq:31_sub} and \eqref{eq:33}, we have
\begin{align}
&\left\|\frac{\sum_{i=1}^k\exp\left(\frac{1}{\sqrt{d}}\bm{x}_1\bm{W}_{\bm{Q}}\bm{W}_{\bm{K}}^{\top}\bm{x}_i^{\top}\right)}{\sum_{j=1}^k\exp\left(\frac{1}{\sqrt{d}}\bm{x}_1\bm{W}_{\bm{Q}}\bm{W}_{\bm{K}}^{\top}\bm{x}_j^{\top}\right)}\bm{x}_{i}\bm{W}_V-\bm{\mu}_1\bm{W}_V\right\|_{\infty}\notag\\
&\le\left(k\exp\left(-\frac{\nu_2}{\sqrt{d}}\right)-k\exp\left(\frac{\nu_4}{\sqrt{d}}\right)\left(\exp\left(\frac{\psi(\delta,d)}{\sqrt{d}}\right) - 1\right)\right)^{-1}\notag\\
&\cdot\left(\exp\left(\frac{\psi(\delta,d)}{\sqrt{d}}\right)\left[2\tau_1(k,k_1)\|\bm{\mu}_1\bm{W}_V\|_{\infty}+\sigma\nu_5 \sqrt{\frac{2\tau_2(k,k_1)}{d}\log\left(\frac{2d}{\delta}\right)}\right]\right.\notag\\
&\left.\quad\quad+(k-k_1)\exp\left(\frac{\nu_2}{\sqrt{d}}\right)\|\bm{\mu}_1\bm{W}_V\|_{\infty}-k_1\exp\left(\frac{\nu_4}{\sqrt{d}}\right)\|\bm{\mu}_1\bm{W}_V\|_{\infty}\right)\notag\\
&\le \frac{\exp\left(\frac{\psi(\delta,d)}{\sqrt{d}}\right)\sigma\nu_5 \sqrt{\exp\left(\frac{-2\nu_4}{\sqrt{d}}\right)\frac{2\tau_2(k,k_1)}{dk^2}\log\left(\frac{2d}{\delta}\right)}}
{1+\exp\left(-\frac{\nu_4+\nu_2}{\sqrt{d}}\right)-\exp\left(\frac{\psi(\delta,d)}{\sqrt{d}}\right)}\notag\\
&+\frac{\exp\left(\frac{\psi(\delta,d)-\nu_4}{\sqrt{d}}\right)\frac{2\tau_1(k,k_1)}{k}-\frac{k_1}{k}\exp\left(\frac{\nu_4}{\sqrt{d}}\right)+(1-\frac{k_1}{k})\exp\left(\frac{\nu_4}{\sqrt{d}}\right)}
{1+\exp\left(-\frac{\nu_4+\nu_2}{\sqrt{d}}\right)-\exp\left(\frac{\psi(\delta,d)}{\sqrt{d}}\right)}\cdot \|\bm{\mu}_1\bm{W}_V\|_{\infty}\notag\\
&\le  \left[8\exp\left(\frac{\nu_2-\nu_4+\psi(\delta,d)}{\sqrt{d}}\right)-\left(7+\exp\left(\frac{\nu_2-\nu_4+\psi(\delta,d)}{\sqrt{d}}\right)\right)\frac{k_1}{k}\right]\|\bm{\mu}_1\bm{W}_V\|_{\infty}\notag\\
&+4\exp\left(\frac{\psi(\delta,d)}{\sqrt{d}}\right)\sigma\nu_5 \sqrt{\frac{2}{dk}\log\left(\frac{2d}{\delta}\right)},\label{eq:36}
\end{align}
where second inequality uses $\nu_4\ge \mu_2$ and last inequality uses the definition of $\tau_1(k,k_1)$, $\tau_2(k,k_1)$ and $\sqrt{d}\ge 3(\psi(\delta,d)+\nu_2+\nu_4)$.
\end{proof}

{\bf Proof of Lemma \ref{lem:2}}

Without loss of generality, we only consider the case with first query patches. In the $k$-NN attention scheme, we first use the dot-product product to compute the similarity between query and each key-patches and then use the $\mathrm{softmax}$ to normalize the similarities. We make the following assumption to facility our analysis.\\

{\bf\noindent Assumption 3.1} There exists $\bm{\beta}^*\in\mathbbm{R}^{1\times n}$ and $\bm{\beta}^*\in \Delta$ such that $\bm{q}_1 = \bm{\beta}\bm{K}+\bm{\epsilon}$, where $\bm{\epsilon}$ is filled with random variable follows $\mathcal{N}(0,\sigma^2)$ for some $\sigma>0$.

To see the connection between the Assumption {\bf 3.1} with attention scheme, we consider the follow problem.
\begin{align}
\min_{\bm{\beta}\in\Delta}\|\bm{K}^{\top}\bm{\beta}^{\top}-\bm{q}_1^{\top}\|_2^2,\label{eq:assumption_3.1_1}
\end{align}
If $\bm{K}$ is normalized with zero columns mean and we apply the exponential gradient method on the initial solution $\bm{\beta}_0 = \frac{1}{n}\bm{e}$ with step length $1/\sqrt{d}$, the one step updated solution $\bm{\beta}_1$ is
\begin{align}
\bm{\beta}_1 = \frac{\exp(\frac{1}{\sqrt{d}} \bm{K}\bm{q}_1^{\top})}{\sum_{i=1}^k\exp(\frac{1}{\sqrt{d}} \bm{k}_i\bm{q}_1^{\top})}.\label{eq:assumption_3.1_2}
\end{align}
The above equation \eqref{eq:assumption_3.1_2} is just the attention scheme used standard transformer type model. Based on Assumption {\bf 3.1}, we can treat $\bm{\beta}_1$ as an approximation of underlying true parameters $\bm{\beta}^*$.

On the other hand, it is commonly believed that only part of patches are correlated with the query patch (i.e., with non-zero similarity weights.) and it would be ideal if we could use a computational cheap method to eliminate the irrelevant patches. In this paper, we consider the top-k selection scheme. To see the rationality of the  top-$k$ selection, we consider augmenting \eqref{eq:assumption_3.1_1} with $L_2$ regularization on $\bm{\beta}$.
\begin{align}
&\min_{\bm{\beta}\in\Delta}\|\bm{K}^{\top}\bm{\beta}^{\top}-\bm{q}_1^{\top}\|_2^2 + \frac{\lambda}{2} \|\bm{\beta}\|_2^2,\label{eq:assumption_3.1_3}\notag\\
\Rightarrow& \bm{K}(\bm{K}^{\top}\bm{\beta}^{\top}+\bm{q}_1^{\top}) + \lambda\bm{\beta}^{\top} + \bm{\lambda_1}+ \bm{e}\lambda_2=0,\notag\\
\Rightarrow& \bm{\beta}^{\top} = (\bm{K}\bm{K}^{\top}+\lambda \bm{I})^{-1}\left(\bm{K}\bm{q}_1^{\top}+ \bm{\lambda_1}+ \bm{e}\lambda_2\right),\notag
\end{align}
where we use the KKT optimal condition, and $\bm{\lambda}_1$, $\lambda_2$ are Lagrange multipliers to make sure $\bm{\beta}\in\Delta$.  If $\lambda$ is large enough and $\bm{\beta}>0$, we will have
\begin{align}
\bm{\beta}^{\top} \approx \frac{1}{\lambda}\left(\bm{K}\bm{q}_1^{\top}+ \bm{e}\lambda_2\right)\propto \bm{K}\bm{q}_1^{\top}
\end{align}
The above result indicates that we may selection the important elements in $\bm{\beta}$ (e.t., with large magnitude) by considering rankness in vector $\bm{K}\bm{q}_1^{\top}$. 

We then discussion the correctness of the top-k selection with the following  regularity assumptions on $\bm{K}$ and $\bm{q}_1$.\\

{\bf \noindent Assumption 3.2}
\begin{enumerate}
	\item $\bm{K}$ is normalized with row zero mean. Let $\bm{\Sigma} = \bm{K}\bm{K}^{\top}$ and $\bm{Z} = \bm{\Sigma^{-1/2}}\bm{K}^{\top}$. We assume there exist some $c, c_4> 1$ and $C_1>0$ such that the following inequality 
	$$\mathbbm{P}\left(\lambda_{\max}(\tilde{p}^{-1}\tilde{\bm{Z}}\tilde{\bm{Z}}^{\top})>c_4 \textrm{ and } \lambda_{\max}(\tilde{p}^{-1}\tilde{\bm{Z}}\tilde{\bm{Z}}^{\top}) <\frac{1}{c_4}\right)<\exp(-C_1d)$$
	holds for any $\tilde{p}\times d$ submatrix $\tilde{\bm{Z}}$ of $\bm{Z}$ with $cn<\tilde{p}\le n$.
	\item $\textrm{var}(\bm{q}_1) = \mathcal{O}(1)$ and for some $\kappa\ge 0$ and $c_5,c_6>0$,
	$$\min_{i:\beta_i^*>0}\beta_i\ge \frac{c_5}{d^{\kappa}}\textrm{ and } \min_{i:\beta_i^*>0}\textrm{cov}(\beta_i^{-1}\bm{q}_1^{\top},\bm{k}_i^{\top})\ge c_6$$
	\item There exist some $\tau \in [0,1) $ and $c_7>0$ such that
	$$\lambda_{\max}(\bm{\Sigma})\le c_7d^{\tau}.$$
\end{enumerate}

\begin{lemma}[formal statement of Lemma \label{lem:21}]
	Let's assume only be $s$ keywords are relevant to the query $l$. Under Assumption {\bf 3.1} and {\bf 3.2}, when $2\kappa + \tau < 1$, with probability $1-\mathcal{O}(s\exp(-Cd^{1-2\kappa}/\log d))$, we have
	\begin{align}
\sum_{i=1}^n\mathbbm{1}(\bm{q}_l\bm{k}_i^{\top}\ge \min_{i\in\mathcal{M}^*}\bm{q}_l\bm{k}_i^{\top})\le cnd^{2\kappa+\tau-1},	
	\end{align}
where $\mathcal{M}^* =\{i: \textrm{keyword $i$ is relevant to the query $l$.}\}$ , and $\tau$, $\kappa$, $c$ and $C$ are positive constants.
\end{lemma}

\begin{proof}
Our strategy is the similar to the proof of Theorem 1 in \cite{fan2008sure}.  	

Based on equation (44) in \cite{fan2008sure}, we have
\begin{align}
\bm{K}\bm{K}^{\top} = n\bm{\Sigma}^{1/2}\tilde{\bm{U}}^{\top}\textrm{diag}(\mu_1,...,\mu_d)	\tilde{\bm{U}}\bm{\Sigma}^{1/2},\label{eq:knn1}
\end{align}
where $\mu_1$,..., $\mu_d$ are $d$ eigenvalues of $p^{-1}\bm{Z}\bm{Z}^{\top}$, $\tilde{\bm{U}} = (I_d,\bm{0})_{d\times n}\bm{U}$, and $\bm{U}$ is uniformally distributed on the orthogonal group $O(n)$. To facility our further analysis we denote $\bm{\omega} = \bm{q}_l\bm{K}^{\top}$. By definition of $\bm{\omega}$ and per Assumption {\bf 3.1}, we have
\begin{align}
\bm{\omega}^{\top}	=\bm{K}\bm{q}_l^{\top} = \bm{K}\bm{K}^{\top}\bm{\beta}^{\top} + \bm{K}\bm{\epsilon}^{\top} \doteq\bm{\xi}+\bm{\eta}.
\end{align}
We then separately study $\bm{\xi}$ and $\bm{\eta}$.

{\bf Analysis on $\bm{\xi}$.} We first bound $\bm{\xi}$ from above. Since $\{\mu_i\}$ is the eigenvalues of $n^{-1}\bm{Z}\bm{Z}^{\top}$, we have
\begin{align}
\textrm{diag}(\mu_1^2,...,\mu_d^2)\le \left[\lambda_{\max}(n^{-1}\bm{Z}\bm{Z}^{\top})\right]^2I_d	
\end{align}
and 
$$\tilde{\bm{U}}\bm{\Sigma}\tilde{\bm{U}}^{\top}\le \lambda_{\max}(\bm{\Sigma})I_d$$
There lead to
\begin{align}
\|\bm{\xi}\|^2\le n^2\lambda_{\max}(\bm{\Sigma})\left[\lambda_{\max}(n^{-1}\bm{Z}\bm{Z}^{\top})\right]^2\cdot\bm{\beta}^{\top}\bm{\Sigma}^{1/2}\tilde{\bm{U}}^{\top}\tilde{\bm{U}}\bm{\Sigma}^{1/2}\bm{\beta}.
\end{align}
Let $Q$ belongs to the orthogonal group $O(n)$ such that $\bm{\Sigma}^{1/2}\bm{\beta} =\|\bm{\Sigma}^{1/2}\bm{\beta}\|Q\bm{e}_1$. Then, it follows from {Lemma 1} in \cite{fan2008sure} that
\begin{align}
	\bm{\beta}^{\top}\bm{\Sigma}^{1/2}\tilde{\bm{U}}^{\top}\tilde{\bm{U}}\bm{\Sigma}^{1/2}\bm{\beta} = \|\bm{\Sigma}^{1/2}\bm{\beta}\|\langle Q^{\top}\bm{S}Q\bm{e}_1,\bm{e}_1\rangle \overset{(d)}{=}\|\bm{\Sigma}^{1/2}\bm{\beta}\|\langle \bm{S}\bm{e}_1,\bm{e}_1\rangle, 
\end{align}
where we use the symbol $\overset{(d)}{=}$ to denote being identical in distribution. By part 3 in Assumption {\bf3.2}, $\|\bm{\Sigma}^{1/2}\bm{\beta}\|^2 = \bm{\beta}^{\top}\bm{\Sigma}\bm{\beta}\le \textrm{var}(\bm{y}) = \mathcal{O}(1)$, and thus via {Lemma 4 }in \cite{fan2008sure}, we have for some $C>0$,
\begin{align}
\mathbbm{P}\left(\bm{\beta}^{\top}\bm{\Sigma}^{1/2}\tilde{\bm{U}}^{\top}\tilde{\bm{U}}\bm{\Sigma}^{1/2}\bm{\beta} >(d/n)\right)\le \mathcal{O}\left(\exp(-Cd)\right).	
\end{align}
Combining with $\lambda_{\max}(\bm{\Sigma}) = O(d^{\tau})$ and $\mathbbm{P}(\lambda_{\max}(n^{-1}\bm{Z}\bm{Z}^{\top})>c_1)\le \exp\left(-C_1d\right)$ by parts 1 and 3 in Assumption {\bf 3.2} along with union bound, we have
\begin{align}
	\mathbbm{P}\left(\|\bm{\xi}\|^2\ge \mathcal{O}(d^{1+\tau}n)\right)\le \mathcal{O}(\exp(-Cd)).
\end{align}

We then consider the lower bound on $\xi_i$ for $i\in \mathcal{M}_{*}$. By \eqref{eq:knn1}, we have
\begin{align}
	\xi_i = n\bm{e}_i^{\top}\bm{\Sigma}^{1/2}\tilde{\bm{U}}^{\top}\textrm{diag}(\mu_1,...,\mu_d)	\tilde{\bm{U}}\bm{\Sigma}^{1/2}\bm{\beta}.
\end{align}
Note that $\|\bm{\Sigma}^{1/2}\bm{e}_i\| = \sqrt{\textrm{var}(\bm{X}_i)} = 1$, $\|\bm{\Sigma}^{1/2}\bm{\beta}\| = \mathcal{O}(1)$. By part 2 of Assumption {\bf 3.2}, there exists some $c>0$ such that
\begin{align}
\langle\bm{\Sigma}^{1/2}\bm{\beta},\bm{\Sigma}^{1/2}\bm{e}_i\rangle = \beta_i\textrm{cov}(\beta_i^{-1}\bm{q}_1^{\top},\bm{k}_i)\ge \frac{c}{d^{\kappa}}.	
\end{align}
Thus, there exists $Q$ in orthogonal group $O(n)$ such that $\bm{\Sigma}^{1/2}\bm{e}_i = Q\bm{e}_1$ and 
\begin{align}
\bm{\Sigma}^{1/2}\bm{\beta} = \langle\bm{\Sigma}^{1/2}\bm{\beta},\bm{\Sigma}^{1/2}\bm{e}_i\rangle Qe_1 + \mathcal{O}(1)Q\bm{e}_2.	
\end{align}
Since $(\mu_1,....,\mu_d)^{\top}$ is independent of $\tilde{\bm{U}}$ by {Lemma 1} in \cite{fan2008sure} and the uniform distribution on the orthogonal group $O(n)$ is invariant under itself, it follows that
\begin{align}
\xi_i \overset{(d)}{=}	\langle\bm{\Sigma}^{1/2}\bm{\beta},\bm{\Sigma}^{1/2}\bm{e}_i\rangle R_1
+ \mathcal{O}(n)R_2 \doteq \xi_{i,1}+\xi_{i,2},
\end{align}
where $\bm{R} = (R_1,R_2,...,R_n)^{\top} = \tilde{\bm{U}}^{\top}\textrm{diag}(\mu_1,...,\mu_d)\tilde{\bm{U}}^{\top}\bm{e}_1$. We will bound the above two terms $\xi_{i,1}$ and $\xi_{i,2}$ separately. One can verify
\begin{align}
R_1 \ge \bm{e}_1^{\top}\tilde{\bm{U}}^{\top}\lambda_{\min}(n^{-1}\bm{Z}\bm{Z}^{\top})I_d\tilde{\bm{U}}\bm{e} = \lambda_{\min}(n^{-1}\bm{Z}\bm{Z}^{\top})\langle\bm{S}\bm{e}_1,\bm{e}_1\rangle,	
\end{align}
and thus by part 1 of assumption {\bf 3.2}, {Lemma 4} in \cite{fan2008sure}, and union bound, we have for some $c,C>0$,
\begin{align}
\mathbbm{P}(R_1<cd/n)\le \mathcal{O}(\exp(-Cd)).	
\end{align}
Therefore we have for some $c>0$,
\begin{align}
\mathbbm{P}(\xi_{i,1}\le cd^{1-\kappa})\le \mathcal{O}(\exp(-Cd))	
\end{align}

Similarly, we can also show that
\begin{align}
\mathbbm{P}(\|\bm{R}\|^2\ge O(d/n))\le \mathcal{O}(\exp(-Cd)).	
\end{align}
Via some analysis, we can show that the distribution of $\tilde{\bm{R}} = (R_2,...,R_n)^{\top}$ is invariant under the orthogonal group $O(n-1)$. Then, it follows that $\tilde{\bm{R}}\overset{(d)}{=}\|\tilde{\bm{R}}\|\bm{W}/\|\bm{W}\|$, where $\bm{W} = (W_1,...,W_{n-1})^{\top}\sim\mathcal{N}(0,I_{n-1})$, independent of $\|\tilde{\bm{R}}\|$. Thus, we have
\begin{align}
R_2 \overset{(d)}{=}\|\tilde{\bm{R}}\|\frac{W_1}{\|\bm{W}\|}.
\end{align}
And via the Lemma 5 in \cite{fan2008sure}, we can show
\begin{align}
\mathbbm{P}(|\xi_{i,2}|\ge c\sqrt{d}|W|)\le \mathcal{O}(\exp(-Cd)),
\end{align}
where $W$ is $\mathcal{N}(0,1)$-distributed random variable. We then pick 
$x_d = c\sqrt{2C}d^{1-\kappa}/\sqrt{\log d}$. Then, by standard tail bound, we have
\begin{align}
\mathbbm{P}(c\sqrt{d}|W|\ge x_d)\le \mathcal{O}(\exp(-Cd^{1-2\kappa}/\log d)),	
\end{align}
Then 
\begin{align}
\mathbbm{P}(|\xi_{i,2}|\ge x_d)\le \mathcal{O}(\exp(-Cd^{1-2\kappa}/\log d)).	
\end{align}
It implies that for $i$ with $\beta_i^*>0$, we have
\begin{align}
&\mathbbm{P}\left(\xi_{i,1}-|\xi_{i,2}|\le cd^{1-\kappa}\right)\le \mathcal{O}(\exp(-Cd^{1-2\kappa}/\log d))\notag\\
\Rightarrow &\mathbbm{P}\left(\xi_{i}\le cd^{1-\kappa}\right)\le \mathcal{O}(\exp(-Cd^{1-2\kappa}/\log d))\notag.
\end{align}
Next, we examine term $\bm{\eta} = (\eta_1,...,\eta_n)^{\top}=\bm{K}\bm{\epsilon}^{\top}$. Clearly, we have
\begin{align}
\bm{K}^{\top}\bm{K} = \bm{Z}\bm{\Sigma}\bm{Z}^{\top}\le \bm{Z}\lambda_{\max}(\bm{\Sigma})I_n\bm{Z}^{\top} = n\lambda_{\max}(\bm{\Sigma})\lambda_{\max}(n^{-1}\bm{Z}\bm{Z}^{\top})I_d.	
\end{align}
Then, it follows that
\begin{align}
\|\bm{\eta}\|^2 = \bm{\epsilon}\bm{K}^{\top}\bm{K}\bm{\epsilon}^{\top}\le n\lambda_{\max}(\bm{\Sigma})\lambda_{\max}(n^{-1}\bm{Z}\bm{Z}^{\top})\|\bm{\epsilon}\|^2.
\end{align}
From Assumption {\bf 3.1}, we know that $\{\epsilon_i^2/\sigma^2\}$ are i.i.d. $\chi_1^{2}$-distributed random variables. Thus there exist $c,C>0$ such that
\begin{align}
\mathbbm{P}(\|\bm{\epsilon}\|^2>cd\sigma^2)\le\exp(-Cd)	
\end{align}
Along with parts 1 and 3 of Assumption {\bf 3.2}, we have
\begin{align}
\mathbbm{P}(\|\bm{\eta}\|^2> O(d^{1+\tau}n))\le \mathcal{O}(\exp(-Cd)).	
\end{align}

We then bound $|\eta_i|$ from above. Given that $\bm{\eta} = K\bm{\epsilon}^{\top}\sim\mathcal{N} (0,\sigma^2KK^{\top})$. Hence $\eta_i|_{K=\bm{K}} \sim\mathcal{N}(0,\textrm{var}(\eta_i|_{K = \bm{K}}))$ with $	\textrm{var}(\eta_i|K = \bm{K}) = \sigma^2\bm{e}_i^{\top}\bm{K}\bm{e}_i$.

Let $\mathcal{E}$ be the event $\{\textrm{var}(\eta_i|\bm{K})\le cd\}$ for some c>0. Then, using the same argument in the previous proof. we can show that
\begin{align}
	\mathbbm{P}(\mathcal{E}^c)\le \mathcal{O}(\exp(-Cd)).
\end{align}
Condition on the event $\mathcal{E}$, for all $x>0$, we have
\begin{align}
\mathbbm{P}(|\eta_i|>x|\bm{K})\le \mathbbm{P}(\sqrt{cd}|W|>x)	
\end{align}
, where $W$ is $\mathcal{N}(0,1)$ random variable. Via union bound, we have
 \begin{align}
 	\mathbbm{P}(|\eta_i|>x)\le\mathcal{O}(\exp(-Cd))+ \mathbbm{P}(\sqrt{cd}|W|>x)
 \end{align}
By setting $x = \sqrt{2cC}d^{1-\kappa}/\sqrt{\log d}$, we have
\begin{align}
	\mathbbm{P}(\sqrt{cn}|W|>x)\le \mathcal{O}(\exp(-Cd^{1-2\kappa}/\log d)) 
\end{align}
Then
\begin{align}
		\mathbbm{P}(|\eta_i|>o(d^{1-\kappa}))\le \mathcal{O}(\exp(-Cd^{1-2\kappa}/\log d)) 
\end{align}
Finally, we could reach
\begin{align}
\mathbbm{P}(\min_{i:\beta^*_i>0}\xi_i -|\eta_i| \le c_1d^{\kappa}) \le \mathcal{O}(s\exp(-Cd^{1-2\kappa}/\log d))\notag\\
\Rightarrow \mathbbm{P}(\min_{i: \beta_i^*>0}\omega_i \le c_1d^{\kappa}) \le \mathcal{O}(s\exp(-Cd^{1-2\kappa}/\log d))\notag
\end{align}

Therefore, with probability $1-\mathcal{O}(s\exp(-Cd^{1-2\kappa}/\log d))$, the magnitudes of $\omega_i$ with $\beta_i^*>0$ are uniformly at least of order $d^{1-\kappa}$ and for some $c>0$, we have
\begin{align}
&\sum_{i=1}^n \mathbbm{1}(\omega_k \ge \min_{i:\beta^*_i>0}\omega_i)	\le \frac{cp}{d^{1-2\kappa-\tau}} \le cnd^{2\kappa +\tau-1}\notag\\
\Rightarrow &\sum_{i=1}^n \mathbbm{1}(\bm{q}_l\bm{k}^{\top}_i \ge \min_{i\in\mathcal{M}_*}\bm{q}_l\bm{k}^{\top}_i)	 \le cnd^{2\kappa +\tau-1}\notag.
\end{align}

\end{proof}

%
%

\end{document}